%% file: LinearQ_Revisit.tex
\documentclass[english]{article}

\usepackage{geometry}
\usepackage[T1]{fontenc}
\usepackage[latin9]{inputenc}
\usepackage{bm}
\usepackage{amsmath}
\usepackage{amssymb}
\usepackage[unicode=true,
 bookmarks=false,
 breaklinks=false,pdfborder={0 0 1},colorlinks=false]
 {hyperref}
\hypersetup{
 colorlinks,citecolor=blue,filecolor=blue,linkcolor=blue,urlcolor=blue}

\makeatletter
\usepackage{amsthm}
\usepackage{cite}  
\usepackage{comment}
\usepackage{natbib}
\usepackage{booktabs,mathtools}

\usepackage{graphicx}
\usepackage[linesnumbered,ruled,vlined]{algorithm2e}
\usepackage{algorithmic}

\usepackage{float}
\usepackage{multirow}
 
\usepackage{dsfont}

\usepackage{color}
\definecolor{yxc}{RGB}{255,0,0}
\definecolor{yjc}{RGB}{125,0,0}
\definecolor{ytw}{RGB}{255,0,127}
\definecolor{gen}{RGB}{0,0,200}

\newcommand{\gen}[1]{\textcolor{gen}{[Gen: #1]}}

\allowdisplaybreaks

\newcommand{\defn}{:=}

\newcommand{\cS}{\mathcal{S}}
\newcommand{\cA}{\mathcal{A}}
\newcommand{\mymid}{\,|\,}
\newcommand{\taugap}{\Delta_{\mathsf{gap}}}

\title{Sample-Efficient Reinforcement Learning Is Feasible for \\ Linearly Realizable MDPs with Limited Revisiting}

\author{Gen Li\thanks{Department of Electrical and Computer Engineering, Princeton University, Princeton, NJ 08544, USA.}  \\
Princeton    \\
\and
	Yuxin Chen\footnotemark[1]\\
 Princeton  \\
	\and
	Yuejie Chi\thanks{Department of Electrical and Computer Engineering, Carnegie Mellon University, Pittsburgh, PA 15213, USA.}\\
	CMU\\
	\and
	Yuantao Gu\thanks{Department of Electronic Engineering, Tsinghua University, Beijing 100084, China.} \\
	Tsinghua   \\
	\and
	Yuting Wei\thanks{Department of Statistics and Data Science, The Wharton School, University of Pennsylvania, Philadelphia, PA 19104, USA.}\\
	UPenn	
	}
\date{May 2021;~~ Revised: October 2021}

\makeatother

\begin{document}

\theoremstyle{plain} \newtheorem{lemma}{\textbf{Lemma}}\newtheorem{proposition}{\textbf{Proposition}}\newtheorem{theorem}{\textbf{Theorem}}\newtheorem{assumption}{\textbf{Assumption}}

\theoremstyle{remark}\newtheorem{remark}{\textbf{Remark}}

\maketitle

\input{abstract}




\input{introduction}

\input{model}

\input{main-results}

\input{related-work}

\input{analysis-gap}

\input{conclusion}

\section*{Acknowledgements}

The authors are grateful to Csaba Szepesv\'ari and Ruosong Wang for helpful discussions about \cite{weisz2021query} and \citet{du2019provably,du2020agnostic}, respectively.  
Y.~Chen is supported in part by the grants AFOSR YIP award FA9550-19-1-0030,
ONR N00014-19-1-2120, ARO YIP award W911NF-20-1-0097, ARO W911NF-18-1-0303, NSF CCF-2106739, CCF-1907661, DMS-2014279 and IIS-1900140,
and the Princeton SEAS Innovation Award.  
Y.~Chi is supported in part
by the grants ONR N00014-18-1-2142 and N00014-19-1-2404, ARO W911NF-18-1-0303,
and NSF CCF-2106778, CCF-1806154 and CCF-2007911. 
Y.~Gu is supported in part by the grant NSFC-61971266.
Y.~Wei is supported in part by the grants NSF CCF-2106778, CCF-2007911 and DMS-2147546/2015447. 
Part of this work was done while Y.~Chen and Y.~Wei were visiting the Simons Institute for the Theory of Computing. 
\appendix

\input{proof-main-lemmas.tex}

\input{auxiliary_lemmas}






\bibliography{bibfileRL}
\bibliographystyle{apalike}

\end{document}

%% file: abstract.tex
\begin{abstract}

Low-complexity models such as linear function representation play a pivotal role in enabling sample-efficient reinforcement learning (RL).  
The current paper pertains to a scenario with value-based linear representation, which postulates linear realizability of the optimal Q-function (also called the ``linear $Q^{\star}$ problem''). While linear realizability alone does not allow for sample-efficient solutions in general, the presence of a large sub-optimality gap is a potential game changer, depending on the sampling mechanism in use.  Informally, sample efficiency is achievable with a large sub-optimality gap when a generative model is available, but is unfortunately infeasible when we turn to standard online RL settings.  

	In this paper, we make progress towards understanding this linear $Q^{\star}$ problem by investigating a new sampling protocol, which draws samples in an online/exploratory fashion but allows one to backtrack and revisit previous states in a {\em controlled} and {\em infrequent} manner. This protocol is more flexible than the standard online RL setting, while being practically relevant and far more restrictive than the generative model. We develop an algorithm tailored to this setting, achieving a sample complexity that scales polynomially with the feature dimension, the horizon, and the inverse sub-optimality gap, but not the size of the state/action space. Our findings underscore the fundamental interplay between sampling protocols and low-complexity function representation in RL.\\   

\end{abstract}

%% file: introduction.tex
\section{Introduction}
\label{sec:intro}

Emerging reinforcement learning (RL) applications necessitate the design of sample-efficient solutions in order to accommodate the explosive growth of problem dimensionality.  
Given that the state space  and the action space  could both be unprecedentedly enormous,   
it is often infeasible to request a sample size exceeding the fundamental limit set forth by the ambient dimension in the tabular setting (which enumerates all combinations of state-action pairs). 
As a result, the quest for sample efficiency cannot be achieved in general without exploiting proper low-complexity structures underlying the problem of interest.

\subsection{Linear function approximation}

Among the studies of low-complexity models for RL,  linear function approximation has attracted a flurry of recent activity,  
mainly due to the promise of dramatic dimension reduction in conjunction with its mathematical tractability (see, e.g., \cite{bertsekas1995neuro,wen2017efficient,yang2019sample,jin2020provably,du2019good} and the references therein). 
Two families of linear function approximation merit particular attention, which we single out below. Here and throughout, we concentrate on a finite-horizon Markov decision process (MDP), 
and denote by $\cS$, $\cA$ and $H$ its state space,  action space, and horizon, respectively. 
\begin{itemize}
	\item {\em Model-based linear representation.}   \citet{yang2019sample,jin2020provably,yang2020reinforcement} studied a stylized scenario when both the probability transition kernel and the reward function of the MDP can be linearly parameterized.  
This type of MDPs is commonly referred to as linear MDPs. Letting $P_h(\cdot \mymid s,a)$ denote the transition probability from state $s$ when action $a$ is executed at time step $h$, the linear MDP model postulates the existence of a set of predetermined $d$-dimensional feature vectors $\{\varphi_h(s,a)\in \mathbb{R}^d\}$ and a collection of unknown parameter matrices $\{{\mu}_h \in \mathbb{R}^{d \times |\cS|}\}$  such that
\begin{align}
	\forall (s,a) \in \cS\times \cA \text{ and } 1\leq h\leq H: \qquad P_h(\cdot \mymid s,a) = \big( \varphi_h(s,a)  \big)^{\top}  {\mu}_h. 
	\label{eq:linear-MDP-intro}
\end{align}
Similar assumptions are imposed on the reward function as well. In other words, the linear MDP model posits that the probability transition matrix is low-rank (with rank at most $d$) with the column space known {\em a priori}, which forms the basis for sample size saving in comparison to the unstructured setting.

	\item {\em Value-based linear realizability.} Rather than relying on linear embedding of the model specification (namely, the transition kernel and the reward function), another class of linear representation assumes that the action-value function (or Q-function) can be well predicted by linear combinations of  known feature vectors $\{\varphi_h(s,a)\in \mathbb{R}^d\}$.  A concrete framework of this kind  assumes linear realizability of the optimal Q-function (denoted by  $Q_h^{\star}$ at time step $h$ from now on), that is, there exist some unknown vectors $\{\theta_h^{\star} \in \mathbb{R}^d\}$ such that
\begin{align}
	\qquad Q_h^{\star}(s,a) = \big\langle \varphi_h(s,a), \theta_h^{\star}  \big\rangle 
	\label{eq:linear-Qstar-intro}
\end{align}
holds for any state-action pair $(s,a)$ and any step $h$. It is self-evident that this framework seeks to represent the optimal Q-function --- which is an $|\cS||\cA|H$-dimensional object --- via $H$ parameter vectors each of dimension $d$.    
 Throughout this work, an MDP obeying this condition is said to be {\em an MDP with linearly realizable optimal Q-function}, or more concisely, {\em an MDP with linear $Q^{\star}$}. 

\end{itemize}

\subsection{Sample size barriers with linearly realizable $Q^{\star}$ }

The current paper focuses on MDPs with linearly realizable optimal Q-function $Q^{\star}$. In some sense, this is arguably the weakest assumption one can hope for; that is, if linear realiability of $Q^{\star}$ does not hold, then a linear function approximation should perhaps not be adopted in the first place.
In stark contrast to linear MDPs that allow for sample-efficient RL (in the sense that the sample complexity is almost independent of $|\cS|$ and $|\cA|$ but instead depends only polynomially on $d,H$ and possibly some sub-optimality gap), MDPs with linear $Q^{\star}$ do not admit a similar level of sample efficiency in general.  
To facilitate discussion, we summarize key existing results for a couple of settings. Here and below, the notation $f(d) = \Omega(g(d))$ means that $f(d)$ is at least on the same order as $g(d)$ when $d$ tends to infinity.   
\begin{itemize}
\item {\em Sample inefficiency under a generative model.}
 Even when a generative model or a simulator is available --- so that the learner can query arbitrary state-action pairs to draw samples from \citep{kearns1999finite} --- one can find a hard MDP instance in this class that requires at least $\min \big\{ \exp(\Omega(d)), \exp(\Omega(H)) \big\}$ samples regardless of the algorithm in use \citep{weisz2021exponential}.  

\item {\em Sample efficiency with a sub-optimality gap under a generative model.}   
The aforementioned sample size barrier can be alleviated if, for each state, there exists a sub-optimality gap $\taugap$ between the value under the optimal action and that under any sub-optimal action. 
As asserted by \citet[Appendix C]{du2019good}, a sample size that scales polynomially in $d$, $H$ and $1/\taugap$ is sufficient to identify the optimal policy, assuming access to the generative model. 

\item {\em Sample inefficiency with a large sub-optimality gap in online RL.} Turning to the standard episodic online RL setting (so that in each episode, the learner is given an initial state and executes the MDP for $H$ steps to obtain a sample trajectory),  sample-efficient algorithms are, much to our surprise, infeasible even in the presence of a large sub-optimality gap. As has been formalized in \citet{wang2021exponential}, it is possible to construct a hard MDP instance with a constant sub-optimality gap that cannot be solved without at least $\min \big\{ \exp(\Omega(d)), \exp(\Omega(H)) \big\}$ samples. 

\end{itemize} 

In conclusion, the linear $Q^{\star}$ assumption, while succinctly capturing the low-dimensional structure, still presents an undesirable hurdle for RL. 
The sampling mechanism commonly studied in standard RL formulations precludes sample-efficient solutions even when a favorable sub-optimality gap comes into existence.

\subsection{Our contributions}

Having observed the exponential separation between the generative model and standard online RL when it comes to the linear $Q^{\star}$ problem,  
one might naturally wonder whether there exist practically relevant sampling mechanisms --- more flexible than standard online RL yet more practical than the generative model --- that promise substantial sample size reduction. This motivates the investigation of the current paper, as summarized below. 
\begin{itemize}
	\item {\em A new sampling protocol: sampling with state revisiting.} We investigate a flexible sampling protocol that is built upon classical online/exploratory formalism but allows for {\em state revisiting} (detailed in Algorithm~\ref{alg:sampling-mechanism}). In each episode, the learner starts by running the MDP for $H$ steps, and is then allowed to revisit any previously visited  state and re-run the MDP from there. The learner is allowed to revisit states for an arbitrary number of times, although executing this feature too often might inevitably incur an overly large sampling burden. This sampling protocol accommodates several realistic scenarios; for instance, it captures the ``save files'' feature in video games that allows players to record player progress and resume from the save points later on. 
		In addition, state revisiting is reminiscent of Monte Carlo Tree Search implemented in various real-world applications, which assumes that the learner can go back to father nodes (i.e., previous states) \citep{silver2016mastering}. This protocol is also referred to as {\em local access to the simulator}
in the recent work \citet{yin2021efficient}.

	\item {\em A sample-efficient algorithm.} Focusing on the above sampling protocol, we propose a value-based method --- called LinQ-LSVI-UCB --- adapted from the LSVI-UCB algorithm \citep{jin2020provably}. The algorithm implements the optimism principle in the face of uncertainty, while harnessing the knowledge of the sub-optimality gap to determine whether to backtrack and revisit states. The proposed algorithm provably achieves a sample complexity that scales polynomially in the feature dimension $d$, the horizon $H$, and the inverse sub-optimality gap $1/\taugap$, but is otherwise independent of the size of the state space and the action space. 
\end{itemize}

%% file: model.tex
\section{Model and assumptions}
\label{sec:setup}

In this section, we present precise problem formulation, notation, as well as a couple of key assumptions. 
Here and throughout, we denote by $|\cS|$ the cardinality of a set $\cS$, and adopt the notation $[n]\coloneqq \{1,\cdots,n\}$.

\subsection{Basics of Markov decision processes}

\paragraph{Finite-horizon MDP.} 

The focus of this paper is the setting of a finite-horizon MDP, 
as represented by the quintuple $\mathcal{M} = (\mathcal{S}, \mathcal{A}, \{P_h\}_{h=1}^H, \{r_h\}_{h=1}^H, H)$ \citep{agarwal2019reinforcement}. 
Here, $\cS\coloneqq \{1,\cdots, |\cS|\}$ represents the state space, $\cA\coloneqq \{1,\cdots, |\cA|\}$ denotes the action space, $H$ indicates the time horizon of the MDP, 
$P_h $ stands for the probability transition kernel at time step $h\in[H]$ 
(namely, $P_h(\cdot \mymid s,a)$ is the transition probability from state $s$ upon execution of action $a$ at step $h$),
whereas $r_h: \cS \times \cA \rightarrow [0,1]$ represents the reward function at step $h$ 
(namely, we denote by $r_h(s,a)$ the immediate reward received at step $h$ when the current state is $s$ and the current action is $a$).  
For simplicity, it is assumed throughout that all rewards $\{r_h(s,a)\}$ are deterministic and reside within the range $[0,1]$. 
Note that our analysis can also be straightforwardly extended to accommodate random rewards, which we omit here for the sake of brevity.


\paragraph{Policy, value function, and Q-function.}

We let $\pi=\{\pi_h\}_{1\leq h\leq H}$ represent a policy or action selection rule. 
For each time step $h$,  $\pi_h$ represents a deterministic mapping from $\cS$ to $\cA$, namely, action $\pi_h(s)$ is taken at step $h$ if the current state is $s$. The value function associated with policy $\pi$ at step $h$ is then defined as the cumulative reward received between steps $h$ and $H$ under this policy, namely, 
\begin{align}
	\forall (s,h)\in \cS \times [H]:\qquad
	V^{\pi}_h (s) \defn \mathbb{E} \left[ \sum_{t=h}^{H} r_t(s_t,a_t ) \,\Big|\, s_t =s \right]. 
	 \label{eq:defn-value-function-h}
\end{align} 
Here, the expectation is taken over the randomness of an MDP trajectory $\{s_t\}_{t=h}^H$ induced by policy $\pi$ (namely, $a_t = \pi_t(s_t)$ and $s_{t+1}\sim P(\cdot\mymid s_t, a_t)$ for any $h\leq t\leq H$). 
%
Similarly, the action-value function (or Q-function) associated with policy $\pi$ is defined as  
 \begin{equation}
	\forall (s, a, h)\in \cS \times \cA \times [H]:\qquad 
	 Q^{\pi}_h(s,a) \defn \mathbb{E} \left[ \sum_{t=h}^{H} r_t(s_t,a_t ) \,\Big|\, s_h =s, a_h = a \right],  
	 \label{eq:defn-action-value-function-h}
\end{equation} 
which resembles the definition~\eqref{eq:defn-value-function-h} except that the action at step $h$ is frozen to be $a$. 
Our normalized reward assumption (i.e., $r_h(s,a)\in [0,1]$) immediately leads to the trivial bounds
\begin{equation}
	\forall (s, a, h)\in \cS \times \cA \times [H]:\qquad
	0\leq V^{\pi}_h (s) \leq H
	\qquad \text{and}\qquad 
	0 \leq Q^{\pi}_h (s, a) \leq H .
\end{equation}

A recurring goal in reinforcement learning is to search for a policy that maximizes the value function and the Q-function. 
For notational simplicity, we define 
the optimal value function $V^{\star}= \{V_h^{\star}\}_{1\leq h\leq H}$ and optimal Q-function $Q^{\star}=\{Q_h^{\star}\}_{1\leq h\leq H}$ respectively as follows
\begin{equation*}
	\forall (s, a, h)\in \cS \times \cA \times [H]: \qquad
	V^{\star}_h(s) \defn \max_{\pi} V^{\pi}_h(s)
	\qquad \text{and} \qquad
	Q^{\star}_h(s,a) \defn \max_{\pi} Q^{\pi}_h(s,a), 
\end{equation*}
with the optimal policy (i.e., the one that maximizes the value function) represented by $\pi^{\star}=\{\pi^{\star}_h\}_{1\leq h\leq H}$.

\subsection{Key assumptions}

\paragraph{Linear realizability of $Q^{\star}$.}

In order to enable significant reduction of sample complexity, it is crucial to exploit proper low-dimensional structure of the problem. 
This paper is built upon linear realizability of the optimal Q-function $Q^{\star}$ as follows. 
\begin{assumption}
	\label{assumption:linear-Qstar}
	Suppose that there exist a collection of pre-determined feature maps 
\begin{align}
	\varphi = (\varphi_h)_{1\leq h\leq H}, \qquad \varphi_h : \cS\times\cA \rightarrow \mathbb{R}^d
	\label{eq:defn-feature-maps}
\end{align}
and a set of unknown vectors $ \theta_h^{\star} \in \mathbb{R}^d $ ($1\leq h\leq H$) such that  
\begin{align}
	\forall (s, a,h) \in \cS \times \cA \times [H]:
	\qquad
	Q_h^{\star}(s, a) = \langle \varphi_h(s, a), \, \theta_h^{\star} \rangle.
\end{align}

In addition, we assume that
\begin{align}
	\label{eq:assumptions-phi-theta-size}
	\forall (s, a,h) \in \cS \times \cA \times [H]: \qquad
	\|\varphi_h(s, a)\|_2 \le 1
	\quad \text{and} \quad
	\|\theta_h^{\star}\|_2 \le 2H\sqrt{d}. 
\end{align}
\end{assumption}
In other words, we assume that $Q^{\star}=\{Q_h^{\star}\}_{1\leq h\leq H}$ can be embedded into a $d$-dimensional subspace encoded by $\varphi$, 
with $d\leq |\cS||\cA|$. 
In fact, we shall often view $d$ as being substantially smaller than the ambient dimension $|\cS||\cA|$ 
in order to capture the dramatic degree of potential dimension reduction. 
It is noteworthy that linear realizability of $Q^{\star}$ in itself is a considerably weaker assumption compared to 
the one commonly assumed for linear MDPs \citep{jin2020provably} (which assumes $\{P_h\}_{1\leq h\leq H}$ and  $\{r_h\}_{1\leq h\leq H}$
are all linearly parameterized). The latter necessarily implies the former, while in contrast the former by no means implies the latter. 
Additionally, we remark that the assumption \eqref{eq:assumptions-phi-theta-size}
is compatible with what is commonly assumed for linear MDPs;  see, e.g., \citet[Lemma~B.1]{jin2020provably} for a reasoning about why this bound makes sense.

\paragraph{Sub-optimality gap.}

As alluded to previously, another metric that comes into play in our theoretical development is the sub-optimality gap. 
Specifically, for each state $s$ and each time step $h$, we define the following metric
\begin{equation}
	\Delta_h(s) \defn \min_{a \notin \cA_s^{\star}} \big\{ V^{\star}_h(s) - Q^{\star}_h(s, a) \big\} 
	\qquad
	\text{with }\cA_s^{\star} \defn \big\{a: Q^{\star}_h(s, a) = V^{\star}_h(s)\big\}.
	\label{eq:defn-gap-h}
\end{equation}
In words, $\Delta_h(s)$ quantifies the gap --- in terms of the resulting Q-values --- between the optimal action and the sub-optimal ones. 
It is worth noting that there might exist multiple optimal actions for a given $(s,h)$ pair, namely,  
the set $\cA_s^{\star}$ is not necessarily a singleton.  
Further, we define the minimum gap over all $(s,h)$ pairs as follows 
\begin{equation}
	\Delta_{\mathsf{gap}} \defn \min_{s, h \,\in\, \cS \times [H]} \Delta_h(s) ,
	\label{eq:defn-gap}
\end{equation}
and refer to it as the sub-optimality gap throughout this paper.

\subsection{RL under sampling with state revisiting}

In standard online episodic RL settings, the learner collects data samples by executing multiple length-$H$ trajectories in the MDP $\mathcal{M}$ via suitably chosen policies; more concretely, in the $n$-th episode with a given initial state $s_0^n$, the agent executes a policy to generate a sample trajectory 
$\{(s_h^n, a_h^n)\}_{1\leq h\leq H}$, where $(s_h^n, a_h^n)$ denotes the state-action pair at time step $h$. 
This setting underscores the importance of trading off exploitation and exploration. 
As pointed out previously, however, this classical sampling mechanism could be highly inefficient for MDPs with linearly realizable $Q^{\star}$, even in the face of a constant sub-optimality gap \citep{wang2021exponential}.   

\paragraph{A new sampling protocol with state revisiting.} 
In order to circumvent this sample complexity barrier, the current paper studies a more flexible sampling mechanism that allows one to revisit previous states in the same episode.  Concretely, in each episode, the sampling process can be carried out in the following fashion: 
\medskip

\begin{algorithm}[H]
\DontPrintSemicolon
	 {\bf Input:} initial state $s_1$. \\
	 Select a policy and sample a length-$H$ trajectory $\{(s_t, a_t)\}_{1\leq t\leq H}$. \\

	\Repeat{the learner terminates it.}{
		Pick any previously visited state $s_h$ in this episode; \\
		Execute a new trajectory starting from $s_h$ all the way up to step $H$, namely, $\{(s_t, a_t)\}_{h \leq t\leq H}$; here, we overload notation to simplify presentation.
	}

    \caption{Sampling protocol for an episode with state revisiting.}
 \label{alg:sampling-mechanism}
\end{algorithm}
\medskip

As a distinguishing feature, the sampling mechanism described in Algorithm~\ref{alg:sampling-mechanism} allows one to revisit previous states and retake samples from there, which reveals more information regarding these states. 
To make apparent its practice relevance, we first note that the generative model proposed in \citet{kearns1999finite,kakade2003sample} --- in which one can query a simulator with  arbitrary state-action pairs to get samples --- is trivially subsumed as a special case of  this sampling mechanism. 
Moving on to a more complicated yet realistic scenario, consider role-playing video games which commonly include built-in ``save files'' features. 
This type of features allows the player to record its progress at any given point, so that it can resume the game from this save point later on.  
In fact, rebooting the game multiple times from a saved point allows an RL algorithm to conduct trial-and-error learning for this particular game point. 

As a worthy note, while revisiting a particular state many times certainly yields information gain about this state, 
it also means that fewer samples can be allocated to other episodes if the total sampling budget is fixed. 
Consequently, how to design intelligent state revisiting schemes in order to optimize sample efficiency requires careful thinking.

\paragraph{Learning protocol and sample efficiency.} We are now ready to describe the learning process --- which consists of $N$ episodes --- and our goal. 
\begin{itemize}
	\item In the $n$-th episode ($1\leq n\leq N$), the learner is given an initial state $s_1^{(n)}$ (assigned by nature), and executes the sampling protocol in Algorithm~\ref{alg:sampling-mechanism} until this episode is terminated. 
	\item At the end of the $n$-th episode, the outcome of the learning process takes the form of a policy $\pi^{(n)}$, which is learned based on all information collected up to the end of this episode.  
\end{itemize}
The quality of the learning outcome $\{\pi^{(n)}\}_{1\leq n\leq N}$ is then measured by the cumulative regret over $N$ episodes as follows:
\begin{align}
	\mathsf{Regret} (N) \coloneqq  
	\sum_{n=1}^N \left( V_1^{\star} \big(s_1^{(n)}\big) - V_1^{\pi^{(n)}} \big(s_1^{(n)}\big) \right) ,
	\label{eq:defn-regret}
\end{align}
which is what we aim to minimize under a given sampling budget. 
More specifically, for any target level $\varepsilon \in [0,H]$, the aim is to achieve
\[
	\frac{1}{N}\mathsf{Regret} (N) \le \varepsilon
\]
regardless of the initial states (which are chosen by nature), 
using a sample size $T$ no larger than $\mathsf{poly}\big( d, H, \frac{1}{\varepsilon}, \frac{1}{\taugap} \big)$ (but independent of $|\cS|$ and $|\cA|$).  Here and throughout, 
$T$ stands for the total number of samples observed in the learning process; for instance, a new trajectory $\{(s_t, a_t)\}_{h\leq t\leq H}$ amounts to $H-h$ new samples. 
Due to the presence of state revisiting, there is a difference between our notions of regret / sample complexity and the ones used in standard online RL, which we shall elaborate on in the next section. 
An RL algorithm capable of achieving this level of sample complexity is declared to be sample-efficient, given that the sample complexity does not scale with the ambient dimension of the problem (which could be enormous in contemporary RL).

\begin{remark}[From average regret to PAC guarantees and optimal policies.]

There is some intimate connection between regret bounds and PAC guarantees that has been pointed out previously (e.g., \cite{jin2018q}). 
For instance, by fixing the initial state distribution to be identical (e.g., $s_1^{(n)}=s$ for all $1\leq n\leq N$) and choosing the output policy $\widehat{\pi}$ uniformly at random from $\{\pi^{(n)}\mid 1\leq n\leq N\}$, one can easily verify that this output policy $\widehat{\pi}$ is $\varepsilon$-optimal for state $s$,
as long as  $\frac{1}{N}\mathsf{Regret}(N) \leq \varepsilon$.
\end{remark}

%% file: main-results.tex
\section{Algorithm and main results}
\label{sec:main-results}

In this section, we put forward an algorithm tailored to the sampling protocol described in Algorithm~\ref{alg:sampling-mechanism}, 
and demonstrate its desired sample efficiency.

%
%
%

%
%
%
%
%
%
%

\subsection{Algorithm} 
\label{sec:algorithm}

Our algorithm design is motivated by the method proposed in \citep{jin2020provably} for linear MDPs --- called {\em least-squares value iteration with upper confidence bounds (LSVI-UCB)} --- which follows the principle of ``optimism in the face of uncertainty''. In what follows, we shall begin by briefly reviewing the key update rules of LSVI-UCB, and then discuss how to adapt it to accommodate MDPs with linearly realizable $Q^{\star}$ when state revisiting is permitted.

\paragraph{Review: LSVI-UCB for linear MDPs.} Let us remind the readers of the setting of linear MDPs. 
It is assumed that there exist  unknown vectors ${\mu}_h(\cdot)= [ \mu_h^{(1)}, \cdots, \mu_h^{(d)} ]^{\top} \in \mathbb{R}^{d\times |\cS|} $ and $w_h \in \mathbb{R}^d$ such that
\begin{align*}
	\forall(s,a,h) \in \cS\times \cA \times [H]:
	\qquad
	P_h(\cdot \mymid s,a) = \big\langle \varphi(s,a), {\mu}_h(\cdot) \big\rangle
	\quad \text{and} \quad
	r_h(s,a) =  \big\langle \varphi(s,a), w_h(s,a) \big\rangle .
\end{align*}
In other words, both the probability transition kernel and the reward function can be linearly represented using the set of feature maps $\{ \varphi(s,a) \}$. 

LSVI-UCB can be viewed as a generalization of the UCBVI algorithm \citep{azar2017minimax} (originally proposed for the tabular setting) to accommodate linear function approximation. 
In each episode, the learner draws a sample trajectory following the greedy policy w.r.t. the current Q-function estimate with UCB exploration; namely, an MDP trajectory $\{(s_h^n, a_h^n)\}_{1\leq h\leq H}$ is observed in the $n$-th episode.  
Working backwards (namely, going from step $H$ all the way back to step $1$), the LSVI-UCB algorithm in the $n$-th episode consists of the following key updates: 
\begin{subequations}
\label{eq:update-Jin}
\begin{align}
\Lambda_{h} & ~~\gets\sum_{i=1}^{n}\varphi(s_{h}^{i},a_{h}^{i}) \varphi(s_{h}^{i},a_{h}^{i})^{\top}+ \lambda I,\label{eq:update-Lambda-h-Jin}\\
\theta_{h} & ~~\gets\Lambda_{h}^{-1}\sum_{i=1}^{n}\varphi(s_{h}^{i},a_{h}^{i})\left\{ r_{h}(s_{h}^{i},a_{h}^{i})+\max_{a}Q_{h+1}(s_{h+1}^{i},a)\right\} ,\label{eq:update-theta-h-Jin}\\
Q_{h}(\cdot,\cdot) & ~~\gets\min\left\{ \big\langle\theta_{h},\varphi(\cdot,\cdot)\big\rangle  + \beta\sqrt{\varphi(\cdot,\cdot)^{\top}\Lambda_{h}^{-1}\varphi(\cdot,\cdot)}, H\right\}, 
\label{eq:update-Q-h-Jin}
\end{align}
\end{subequations}
with the regularization parameter $\lambda$ set to be $1.$ 
Informally speaking,  $\theta_{h}$ (cf.~\eqref{eq:update-theta-h-Jin}) corresponds to the solution to a ridge-regularized least-squares problem --- tailored to solving the Bellman optimality equation with linear parameterization --- using all samples collected so far for step $h$, 
whereas the matrix $\Lambda_{h}$ (cf.~\eqref{eq:update-Lambda-h-Jin}) captures the (properly regularized) covariance of $\varphi(\cdot,\cdot)$ associated with these samples. 
In particular,  $\big\langle\theta_{h},\varphi(\cdot,\cdot)\big\rangle$ attempts to estimate the Q-function by exploiting its linear representation for this setting, and the algorithm augments it by an upper confidence bound (UCB) bonus $\beta\sqrt{\varphi(\cdot,\cdot)^{\top}\Lambda_{h}^{-1}\varphi(\cdot,\cdot)}$ --- a term commonly arising in the linear bandit literature \citep{lattimore2020bandit} --- to promote exploration, where $\beta$ is a hyper-parameter to control the level of exploration. 
As a minor remark, the update rule \eqref{eq:update-Q-h-Jin} also ensures that the Q-function estimate never exceeds the trivial upper bound $H$.

\paragraph{Our algorithm: LinQ-LSVI-UCB for linearly realizable $Q^{\star}$.} 

Moving from linear MDPs to MDPs with linear $Q^{\star}$, we need to make proper modification of the algorithm. 
To facilitate discussion, let us introduce some helpful concepts. 
\begin{itemize}
\item
	Whenever we start a new episode or revisit a state (and draw samples thereafter), we say that {\em a new path} is being collected. 
The total number of paths we have collected is denoted by $K$. 
\item For each $k$ and each step $h$, we define a set of indices 
\begin{align}
	\mathcal{I}_{h}^k \coloneqq \left\{ i:\; 1\leq i\leq k \mid  \theta_h^i \text{ is updated in the $i$-th path at time step }h \right\}, 
	\label{eq:defn-Ihk}
\end{align}
which will also be described precisely in Algorithm~\ref{alg:linearQ}. 
As we shall see, the cardinality of $\mathcal{I}_{h}^k$ is equal to the total number of new samples that have been collected at time step $h$ up to the $k$-th path. 
\end{itemize}

We are now ready to describe the proposed algorithm. For the $k$-th path, our algorithm proceeds as follows.
\begin{itemize}
	\item {\em Sampling.} Suppose that we start from a state $s_h^k$ at time step $h$.  
		The learner adopts the greedy policy $\pi^k=\{\pi^k_j\}_{h\leq j\leq H}$ in accordance with the current Q-estimate $\{ Q_j^{k-1} \}_{h\leq j\leq H}$, and observes a fresh sample trajectory $\{(s_j^k,a_j^k)\}_{h\leq j\leq H}$ as follows: for $j = h, h+1, \ldots, H$,
\begin{align}
	s_{j+1}^k \sim P_j(\cdot \mymid s_j^k,a_j^k)\qquad\text{with}\qquad a_j^k = \pi^k_j(s_j^k) \coloneqq \arg\max_a Q_j^{k-1}(s_j^k, a).
\end{align}

	\item {\em Backtrack and update estimates.} We then work backwards to update our Q-estimates and the $\theta$-estimates (i.e., estimates for the linear representation of $Q^{\star}$), until the UCB bonus term (which reflects the estimated uncertainty level of the Q-estimate) drops below a threshold determined by the sub-optimality gap $\taugap$. More precisely, working backwards from $h=H$, we carry out the following calculations if certain conditions (to be described shortly) are met: 
\begin{subequations}
\label{eq:updates-new-algorithm}
\begin{align}
\Lambda_h^{k} ~& \gets \sum_{i \in \mathcal{I}_{h}^k} \varphi_h(s_h^i, a_h^i)\varphi_h(s_h^i, a_h^i)^{\top} +  I, \label{eqn:def-Lambda-hk}\\
	\theta_h^{k} ~& \gets \big(\Lambda_h^{k}\big)^{-1} \sum_{i \in \mathcal{I}_{h}^k} \varphi_h(s_h^i, a_h^i)
	\Big\{ r_h(s_h^i, a_h^i) + \big\langle \varphi_{h+1}(s_{h+1}^i, a_{h+1}^i), \, \theta_{h+1}^{k} \big\rangle \Big\}, \label{eqn:update-theta}\\
	b_h^{k}(\cdot, \cdot) ~& \gets \beta \sqrt{\varphi_h(\cdot, \cdot)^{\top}\big(\Lambda_h^{k}\big)^{-1}\varphi_h(\cdot, \cdot)} , 
	\label{eqn:def-b-hk} \\
	Q_h^{k}(\cdot, \cdot) ~&\gets  \min\Big\{ \big\langle \varphi_h(\cdot, \cdot), \theta_h^{k} \big\rangle + b_h^{k}(\cdot, \cdot), \, H \Big\}.  
	\label{eqn:update-Q}
\end{align}
\end{subequations}
Here, we employ the pre-factor 
\begin{align}
	\beta = c_{\beta} \sqrt{dH^4\log \frac{KH}{\delta}}
	\label{eq:defn-beta}
\end{align}
to adjust the level of ``optimism'',
where $c_{\beta}>0$ is taken to be some suitably large constant. 
Crucially, whether the update \eqref{eqn:update-theta} --- and hence \eqref{eqn:update-Q} --- is executed depends on the size of the bonus term $b_{h+1}^{k-1}$ of the last attempt at step $h+1$.  Informally, if the bonus term is sufficiently small compared to the sub-optimality gap, 
		then we have confidence that the policy estimate (after time step $h$) can be trusted in the sense that it is guaranteed to generalize and perform well on unseen states.   

\end{itemize}
 The complete algorithm is summarized in Algorithm~\ref{alg:linearQ}, with some basic auxiliary functions provided in Algorithm~\ref{alg:linearQ-auxiliary}. To facilitate understanding, an illustration is provided in Figure~\ref{fig:illustration}.

Two immediate remarks are in order. In comparison to LSVI-UCB in \citet{jin2020provably}, 
the update rule \eqref{eqn:update-theta} for $\theta_h^k$ employs the linear representation $\big\langle \varphi_{h+1}(s_{h+1}^i, a_{h+1}^i), \, \theta_{h+1}^{k} \big\rangle $ {\em without the UCB bonus} as the Q-value estimate. This subtle difference turns out to be important in the analysis for MDPs with linear $Q^{\star}$. In addition, LSVI-UCB is equivalent to first obtaining a  linear representation of transition kernel $P$ \citep{agarwal2019reinforcement} and then using it to build Q-function estimates and draw samples. In contrast, Algorithm~\ref{alg:linearQ} cannot be interpreted as a decoupling of model estimation and planning/exploration stage, and is intrinsically a value-based approach.



\begin{algorithm}[t]
\DontPrintSemicolon
	\textbf{inputs:} number of episodes $N$, sub-optimality gap $\taugap$, initial states $\{s_1^{(n)}\}_{1\leq n \leq N}$. \\
  \textbf{initialization:} $h=0$, $n=0$, $\theta_j^0=0$ and $\mathcal{I}_{j}^0 = \emptyset$ for all $1\le j \le H$. \\

   \For{$k=1,2,\cdots$}
	{
		
%
	

		call $\pi^{k}\gets$ {\tt get-policy()}. \\
		$h \gets h + 1$.

		\If{$h = 1$} 
		{
			$\pi^{(n)} \gets \pi^{k-1}.$ {\color{blue}\tcp{record the up-to-date policy learned in this episode.}} 
			$n\gets n+1$. {\color{blue}\tcp{start a new episode.}} 
			\If{$n > N$}
			{ $K \gets k-1$ and \Return. {\color{blue}\tcp{terminate after $N$ episodes.}}}
			Set the initial state $s_1^k = s_1^{(n)}$. \\			
		}

		call {\tt sampling()}.  {\color{blue}\tcp{collect new samples from step $h$; see Algorithm~\ref{alg:linearQ-auxiliary}.}}

		{\color{blue}\tcc{backtrack and determine whether to revisit a state and redraw new samples.}} 
		Set $\theta_{H+1}^k = 0$ and $h = H$.  \\
		\While{$h>0~\text{and}~b_{h+1}^{k-1}(s_{h+1}^k, a_{h+1}^k) < \taugap /2 $ (cf.~\eqref{eqn:def-b-hk}) \label{line:select-b}}  
		{
			$\mathcal{I}_{h}^k \gets \mathcal{I}_{h}^{k-1} \bigcup \left\{k\right\}$. {\color{blue}\tcp{expand $\mathcal{I}_{h}^k$ whenever we need to update $\theta_h^k$.}}
			Update $\theta_h^k$ according to \eqref{eqn:update-theta}. \\
			$h \gets h -1$. \\
		}
		call {\tt update-remaining()}.  {\color{blue}\tcp{keep remaining iterates unchanged; see Algorithm~\ref{alg:linearQ-auxiliary}.}} 

	}


	\caption{LinQ-LSVI-UCB with state revisiting. }
 \label{alg:linearQ}
\end{algorithm}

\begin{algorithm}[ht]
\DontPrintSemicolon

\SetKwFunction{FMain}{sampling}
  \SetKwProg{Fn}{Function}{:}{}
  \Fn{\FMain{}}{
	  {\color{blue}\tcc{sampling from the beginning  (if $h=0$) or from a revisited state (if $h>0$). }}
		
		{\color{blue}\tcc{do not update samples prior to step $h$. }}
		\For{$j=1, 2, \ldots, h-1$}
		{
			Set $a_j^k = a_j^{k-1}$, and $s_{j+1}^k = s_{j+1}^{k-1}$. 
		}

		{\color{blue}\tcc{a new round of sampling from step $h$.}}
		\For{$j=h, h+1, \cdots,H$}
		{
			Compute $Q_j^{k-1}(s_j^k, a)$ according to \eqref{eqn:update-Q}. \\
			Take $a_j^k = \arg\max_a Q_j^{k-1}(s_j^k, a)$, and draw $s_{j+1}^k \sim P_j(\cdot\mymid s_j^k, a_j^k)$. 
		}
	
  }

\SetKwFunction{FMain}{update-remaining}
\SetKwProg{Fn}{Function}{:}{}
  \Fn{\FMain{}}{
		{\color{blue}\tcc{keep the estimates prior to step $h$ unchanged. }}
		\For{$j=1, 2, \ldots, h$}
		{
			Set $\theta_j^k = \theta_j^{k-1}$. 
		}
  }

\SetKwFunction{FMain}{get-policy}
\SetKwProg{Fn}{Function}{:}{}
  \Fn{\FMain{}}{
		{\color{blue}\tcc{update the Q-estimates. }}
		\For{$1\leq h\leq H$}
		{
			Set $Q_h^{k-1}(\cdot,\cdot)$ according to \eqref{eqn:update-Q}. \\
			$\pi_h^k (\cdot) \gets \arg\max_a Q_h^{k-1}(\cdot,a)$. 
			
		}
  }

	\caption{Simple auxiliary functions.}
 \label{alg:linearQ-auxiliary}
\end{algorithm}

\begin{figure}[ht]
	\centering
	\includegraphics[width=0.8\textwidth]{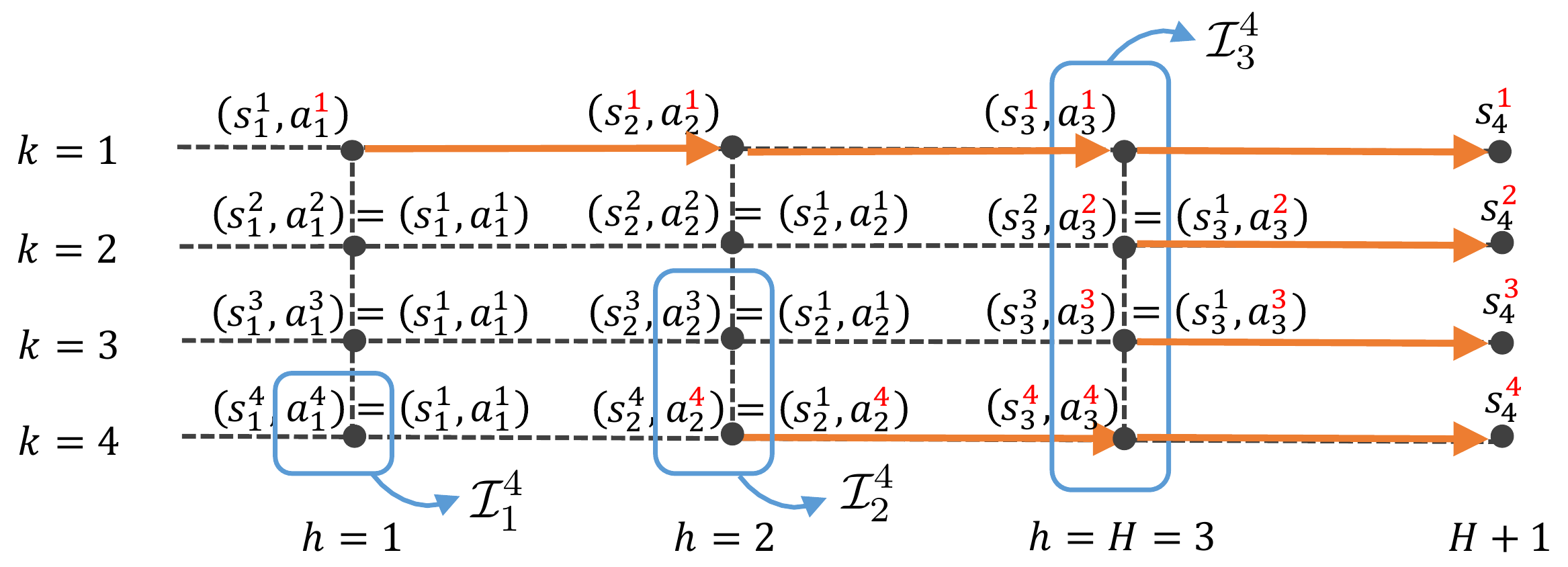}
	\caption{Illustration of the set $\mathcal{I}_{h}^K$ for a simple scenario where $N=1$, $H= 3$, and the total number of paths is $K=4$. 
	After the 1st path, the sampling process revisits state $s_3^1$ twice (each time drawing one new sample), and then revisits state $s_2^1$ to draw two samples from there. This episode then terminates as the conditions are met for all steps. A red superscript indicates new state or new actions taken, and the orange line illustrates the sampling process.  
	Here, we have $\mathcal{I}_{3}^4 = \{1,2,3,4\}$, $\mathcal{I}_{2}^4 = \{3,4\}$ and $\mathcal{I}_{1}^4 = \{4\}$, which record the paths where the linear representations are updated for each respective step. 
	}\label{fig:illustration}
\end{figure}

\subsection{Theoretical guarantees} 
\label{sec:algorithm}


Equipped with the precise description of Algorithm~\ref{alg:linearQ}, 
we are now positioned to present its regret bound and sample complexity analysis.  
Our result is this:

\begin{theorem} \label{thm:main}
	Suppose that Assumption~\ref{assumption:linear-Qstar} holds, and that $c_{\beta}\geq 8$ is some fixed constant. 
	Then for any $0<\delta < 1$, and any initial states $\big\{ s_1^{(n)} \big\}_{1\leq n\leq N}$, 
Algorithm~\ref{alg:linearQ} achieves a regret  (see \eqref{eq:defn-regret}) obeying
\begin{align}
\label{eqn:piazzolla}
	 \frac{1}{N} \mathsf{Regret} (N) \leq 8 c_{\beta}\sqrt{\frac{d^{2}H^{7}\log^{2}\frac{HT}{\delta}}{T}}
\end{align}
with probability at least $1-\delta$, 
provided that $N \geq \frac{4c_{\beta}^{2}d^{2}H^{5}\log^{2}\frac{HT}{\delta}}{\taugap^{2}}$. %
In addition, the total number of state revisits satisfies
\begin{align}
	K - N \le \frac{4c_{\beta}^2d^2H^5\log^2 \frac{KH}{\delta}}{\taugap^2}.
	\label{eq:UB-state-revisits}
\end{align}
\end{theorem}

Theorem~\ref{thm:main} characterizes the sample efficiency of the proposed algorithm. Specifically, the theorem implies that for any $0<\varepsilon < H$, the average regret satisfies
$$\frac{1}{N}\mathsf{Regret}(N) \leq \varepsilon$$ 
with probability exceeding $1-\delta$, once the total number $T$ of samples exceeds 
\begin{align}
	T\geq \frac{64c_{\beta}^2 d^2 H^7 \log^2 \frac{HT}{\delta}}{ \varepsilon^2} .
	\label{eq:LB-sample-complexity}
\end{align}
%
Several implications and further discussions of this result are in order.

\paragraph{Efficiency of the proposed algorithm.} We first highlight several benefits of the proposed algorithm.  
\begin{itemize}

	\item {\em Sample efficiency.}  While our sample complexity bound \eqref{eq:LB-sample-complexity} scales as a polynomial function of both $d$ and $H$, it does not rely on either the state space size $|\cS|$ or the action space size $|\cA|$. This hints at the dramatic sample size reduction when $|\cS||\cA|$ far exceeds the feature dimension $d$ and the horizon $H$. 

	\item {\em A small number of state revisits.}  Theorem~\ref{thm:main} develops an upper bound \eqref{eq:UB-state-revisits} on the total number of state revisits, which is gap-dependent but otherwise independent of the target accuracy level $\varepsilon$. As a consequence, as the sample size $T$ increases (or when $\varepsilon$ decreases), the ratio of the number of state revisits to the sample size becomes vanishingly small,  meaning that the true sampling process in our algorithm becomes increasingly closer to the standard online RL setting.

	\item {\em Computational complexity and memory complexity.} 
		The computational bottleneck of the proposed algorithm lies in the update of $\theta_h^k$ and $b_h^k$ (see \eqref{eqn:update-theta} and \eqref{eqn:def-b-hk}, respectively), which consists of solving a linear systems of equations and can be accomplished using, say, the conjugate gradient method in time on the order of $d^2$ (up to logarithmic factor). In addition, one needs to search over all actions when drawing samples, so the algorithm necessarily depends on $|\cA|$. In total, the algorithm has a runtime  no larger than $\widetilde{O}(d^2|\mathcal{A}|T)$. 
		In addition, implementing our algorithm requires $O(d^2H)$ units of memory. 

\end{itemize}

\paragraph{Cumulative regret over paths.}
Due to the introduction of state revisiting, there are two possible ways to accumulate regrets: over the episodes or over the paths. While our analysis so far adopts the former (see \eqref{eq:defn-regret}),  it is not difficult to translate our regret bound over the episodes to the one over the paths. To be more precise, let us denote the regret over paths as follows for distinguishing purposes: 
\begin{align}
	\mathsf{Regret}_{\mathsf{path}} (K) \coloneqq  \sum_{k=1}^K \left( V_1^{\star} \big(s_1^{k}\big) - V_1^{\pi^{k}} \big(s_1^{k}\big) \right) .
	\label{eq:standard-regret}
\end{align}
A close inspection of our analysis readily reveals the following regret upper bound
\begin{subequations} 
\label{eq:regret_path}	
\begin{align} 
	\mathsf{Regret}_{\mathsf{path}} (K)  &\le 4c_{\beta}\sqrt{d^2H^6K\log^2\frac{HT}{\delta}} + \frac{4c_{\beta}^2d^2H^6\log^2 \frac{KH}{\delta}}{\Delta_{\mathsf{gap}}^2}  \label{eq:worst_case}
\end{align}
with probability exceeding $1-\delta$; see Section~\ref{sec:analysis-proof-regret-path} for details. 
This bound confirms that the regret over the paths exhibits a scaling of at most $\sqrt{K}$.

\paragraph{Logarithmic regret.}

As it turns out, our analysis further leads to a significantly strengthened upper bound on the expected regret. 
As we shall solidify in Section~\ref{sec:analysis-proof-regret-path}, 
the regret incurred by our algorithm satisfies the following upper bound
\begin{align}
 	\mathbb{E} \big[ \mathsf{Regret} (N) \big] 
	\le \mathbb{E} \big[ \mathsf{Regret}_{\mathsf{path}} (K) \big] 
	\le \frac{17c_{\beta}^{2}d^{2}H^{7}\log^{2}(KH)}{\Delta_{\mathsf{gap}}^{2}},
	\label{eq:expected_regret}
\end{align}
\end{subequations}
largely owing to the presence of the gap assumption. 
This implies that the expected regret scales only logarithmically in the number of paths $K$, 
which could often be much smaller than the previous bound \eqref{eq:worst_case}. 
In fact, this is consistent with the recent literature regarding logarithmic regrets under suitable gap assumptions (e.g., \cite{simchowitz2019non,yang2021q}).

\paragraph{Comparison to the case with a generative model.} We find it helpful to compare our findings with the algorithm developed in the presence of a generative model.   
In a nutshell, the algorithm described in \citet[Appendix C]{du2019good} starts by identifying a ``well-behaved'' basis of the feature vectors, and then queries the generative model to sample the state-action pairs related to this basis. In contrast, our sampling protocol (cf.~Algorithm~\ref{alg:sampling-mechanism}) is substantially more restrictive and does not give us the freedom to sample such a basis. In fact, our algorithm is exploratory in nature, which is more challenging to analyze than the case with a generative model.  


We shall also take a moment to point out a key technical difference between our approach and the algorithm put forward in \citet{du2019good}. 
A key insight in \citet{du2019good} is that: by sampling each {\em anchor} state-action pair for $\mathsf{poly}(1/\taugap)$ times, one 
can guarantee sufficiently accurate Q-estimates in all state-action pairs, which in turn ensures $\pi_k = \pi^\star$  for all state-action pairs in all future estimates.   
This, however, is not guaranteed in our algorithm when it comes to the state revisiting setting. 
Fortunately, the gap condition helps ensure that there are at most $\mathsf{poly}(1/\taugap)$ number of samples such that $\pi_k \ne \pi^\star$, although the discrepancy might happen at any time throughout the execution of the algorithm (rather than only happening at the beginning). In addition, careful use of state revisiting helps avoid these sub-optimal estimates by resetting for at most $\mathsf{poly}(1/\taugap)$ times, which effectively prevents error blowup.
		

\paragraph{Comparison to prior works in the presence of state revisiting.} Upon closer examination, the sampling mechanism of \citet{weisz2021query} considers another kind of state revisiting strategy and turns out to be quite similar to ours, which accesses a batch of samples $\{(s_h, a_h, s_{h+1}^i)\}_{i \ge 1}$ for the current state $s_h$ with all actions $a_h \in \cA$. 
Assuming only $V^{\star}$ is linearly realizable, their sample complexity is on the order of $(dH)^{|\mathcal{A}|}$, 
and hence its sample efficiency depends highly on the condition that $\cA=O(1)$.  
Additionally, \citet{du2020agnostic} proposed an algorithm --- tailored to a setting with deterministic transitions --- that requires sampling each visited state multiple times (and hence can be accomplished when state revisiting is permitted); this algorithm might be extendable to accommodate stochastic transitions.



%% file: related-work.tex
\section{Additional related works}
\label{sec:related-works}

Non-asymptotic sample complexity guarantees for RL algorithms have been studied extensively  in the tabular setting over recent years, e.g., \citet{azar2013minimax,jaksch2010near,azar2017minimax,osband2016generalization,even2003learning,dann2015sample,sidford2018variance,zhang2020almost,li2020breaking,agarwal2019optimality,yang2021q,li2020sample,li2021breaking,li2021tightening,wainwright2019stochastic,agarwal2020optimality,cen2020fast}, which have been, to a large extent, well-understood. The sample complexity typically scales at least linearly with respect to the state space size $|\cS|$ and the action space size $|\cA|$, and therefore, falls short of being sample-efficient when the state/action space is of astronomical size. 
In contrast, theoretical investigation of RL with function approximations is still in its infancy due to the complicated interaction between the dynamic of the MDP with the function class.  
In fact, theoretical support remains highly inadequate even when it comes to linear function approximation. 
For example, plain Q-learning algorithms coupled with linear approximation might easily diverge \citep{baird1995residual}. 
It is thus of paramount interest to investigate how to design algorithms that can efficiently exploit the low-dimensional structure without compromising learning accuracy. 
In what follows, we shall discuss some of the most relevant results to ours. The reader is also referred to the summaries of recent literature in \citet{du2019good,du2021bilinear}.

\paragraph{Linear MDP.} \citet{yang2019sample, jin2020provably} proposed the linear MDP model, which can be regarded as a generalization of the linear bandit model \citep{abbasi2011improved,dimakopoulou2019balanced} and has attracted enormous recent activity (see e.g., \citet{wang2019optimism,yang2020reinforcement,zanette2020frequentist,he2020logarithmic,du2019good,  wang2020reward, hao2020sparse, wang2021sample, wei2021learning, touati2020efficient} and the references therein). 
The results have further been generalized to scenarios with much larger feature dimension by exploiting proper kernel function approximation  
\citep{yang2020bridging,long20212}.

\paragraph{From completeness to realizability.} \citet{du2019good} considered the policy completeness assumption, which assumes that the Q-functions of {\em all policies} reside within a function class that contains all functions that are linearly representable in a known low-dimensional feature space. In particular, \cite{du2019good,lattimore2020learning} examined how the model misspecification error propagates and impacts the sample efficiency of policy learning. A related line of works assumed that the linear function class is closed or has low approximation error under the Bellman operator, referred to as low inherent Bellman error \citep{munos2005error,shariff2020efficient,zanette2019limiting,zanette2020learning}. 

These assumptions remain much stronger than the  {\em realizability} assumption considered herein, where only the optimal Q-function $Q^{\star}$ is assumed to be linearly representable. \cite{wen2017efficient,du2020agnostic} showed that sample-efficient RL is feasible in deterministic systems, which has been extended to stochastic systems with low variance in \cite{du2019provably} under additional gap assumptions. In addition, \citet{weisz2021exponential} established exponential sample complexity lower bounds under the generative model when only $Q^{\star}$ is linearly realizable; their construction critically relied on making the action set exponentially large. When restricted to a constant-size action space, \citet{weisz2021query} provided a sample-efficient algorithm when only $V^{\star}$ is linearly realizable, where their sampling protocol essentially matches ours. Recently, \citet{du2021bilinear} introduced the bilinear class and proposed sample-efficient algorithms when both $V^{\star}$ and $Q^{\star}$ are linearly realizable in the online setting. 



\paragraph{Beyond linear function approximation.} 
Moving beyond linear function approximation, another line of works \citep{ayoub2020model,zhou2020nearly} investigated mixtures of linear MDPs. Moreover, additional efforts have been dedicated to studying the low-rank MDP model (without knowing {\em a priori} the feature space), which aims to learn the low-dimensional features as part of the RL problem; partial examples include \cite{agarwal2020flambe,modi2021model}. We conclude by mentioning in passing other attempts in identifying tractable families of MDPs with structural assumptions, such as \cite{jin2021bellman,jiang2017contextual,wang2020reinforcement,osband2014model}.

%% file: analysis-gap.tex
\section{Analysis}
\label{sec:analysis}
%
Before proceeding, we introduce several convenient notation to be used throughout the proof.
As before, the total number of paths that have been sampled is denoted by $K$. 
For any $(s,a,h) \in \cS\times \cA \times [H]$, we abbreviate 
\begin{align}
	P_{h, s,a} \defn P_h(\cdot \mymid s,a) \in \mathbb{R}^{|\cS|}. 
	\label{eq:defn-P-hsa}
\end{align}
For any time step $h$ in the $k$-th path,
we define the empirical distribution vector $P_h^k \in \mathbb{R}^{|\cS|}$  such that
\begin{align}
	P_h^k(s) \defn \begin{cases} 1, &\text{if }s = s_{h+1}^k; \\ 0,  &\text{if }s \ne s_{h+1}^k. \end{cases} 
	\label{eq:defn-Ph-k-s}
\end{align}
The value function estimate $V_{h}^k: \cS \rightarrow \mathbb{R}$ at time step $h$ after observing the $k$-th path  is defined as
\begin{align}
 \label{defn:V-Q}
	\forall (s,h,k) \in \cS \times [H] \times [K]: \qquad V_h^{k}(s) \coloneqq  \max_{a\in \cA} Q_h^{k}(s,a) ,
\end{align}
where the iterate $Q_h^{k}$ is defined in~\eqref{eqn:update-Q}.

Further, we remind the reader the crucial notation $\mathcal{I}_h^{k}$ introduced in \eqref{eq:defn-Ihk}, 
which represents the set of paths between the 1st and the $k$-th paths that update the estimate of $\theta_h^{\star}$. 
We have the following basic facts. 
\begin{lemma}
\label{lem:basic-facts-Ik}
For all $1 \le k \le K$ and $1 \le h \le H$, one has
\begin{align} \label{eq:I-order}
\mathcal{I}_h^{k} \subseteq \mathcal{I}_{h+1}^{k}. 
\end{align}
In addition,
\begin{align} \label{eq:I-end}
\left|\mathcal{I}_1^{K}\right| = N\qquad\text{and}\qquad\left|\mathcal{I}_H^{K}\right| = K.
\end{align}
\end{lemma}
\begin{proof}
This lemma is somewhat self-evident from our construction, and hence we only provide brief explanation. 
The first claim~\eqref{eq:I-order} holds true since if $\theta_h^i$ is updated in the $i$-th path, then $\theta_j^i$ ($j \ge h$) must also be updated.  
	The second claim~\eqref{eq:I-end} arises immediately from the definition of $N$ (i.e., the number of episodes) and $K$ (i.e., the number of paths).
\end{proof}

\subsection{Main steps for proving Theorem~\ref{thm:main}}

In order to bound the regret for our proposed estimate, we first make note of an elementary relation that follows immediately from our construction: 
\begin{align}
	\sum_{n =1}^N \Big( V_1^{\star}(s_1^{(n)}) - V_1^{\pi^{(n)}}(s_1^{(n)}) \Big)
	 = \sum_{k \in \mathcal{I}_{1}^K} \Big( V_1^{\star}(s_1^k) - V_1^{\pi^k}(s_1^k) \Big),
	 \label{eq:equiv-n-k-regret}
\end{align}
%
%
where we recall the definition of $\mathcal{I}_{1}^K$ in~\eqref{eq:defn-Ihk}.
It thus comes down to bounding the right-hand side of \eqref{eq:equiv-n-k-regret}.

\paragraph{Step 1: showing that $Q_h^k$ is an optimistic view of $Q_h^{\star}$.}

Before proceeding, let us first develop a sandwich bound pertaining to the estimate error of the estimate $Q_h^k$ delivered by Algorithm~\ref{alg:linearQ}.
The proof of this result is postponed to Section~\ref{sec:pf-lem-qk}. 
\begin{lemma} \label{lem:Qk-error}
Suppose that $c_{\beta}\geq 8$. With probability at least $1-\delta$, the following bound
\begin{equation}
0 \le Q_h^k(s, a) - Q_h^{\star}(s, a) \le 2 b_h^k(s, a) \label{eq:Qk-error}
\end{equation}
	holds simultaneously for all  $(s, a,k,h) \in \cS \times \cA\times [K] \times [H]$.
\end{lemma}

In words, this lemma makes apparent that $Q_h^k$ is an over-estimate of $Q_h^{\star}$, with the estimation error dominated by the UCB bonus term $b_h^k$. 
This lemma forms the basis of the optimism principle.

\paragraph{Step 2: bounding the term on the right-hand side of \eqref{eq:equiv-n-k-regret}.}

To control the difference $V_1^{\star}(s_1^k) - V_1^{\pi^k}(s_1^k)$ for each $k$, we establish the following two properties. 
First, combining Lemma~\ref{lem:Qk-error} with the definition~\eqref{defn:V-Q} (i.e., $V_h^{k-1}(s_h^k) = \max_a Q^{k-1}_h(s^k_h, a)$), 
one can easily see that  
\begin{align}
	V_h^{k-1}(s_h^k) \ge Q^{k-1}_h \big( s^k_h, \pi^{\star}(s^k_h, h) \big) 
	\ge Q^{\star}_h \big(s^k_h, \pi_h^{\star}(s^k_h) \big) = V^{\star}_h(s^k_h),
\end{align}
namely, $V_h^{k-1}$ is an over-estimate of $V^{\star}_h$. 
In addition, from the definition $ a^k_h = \pi_h^k(s_h^k) = \arg\max_a Q^{k-1}_h(s^k_h, a)$, one can decompose the difference $V_h^{k-1}(s_h^k) - V_h^{\pi^k}(s_h^k)$ as follows
\begin{align}
\notag &V_h^{k-1}(s_h^k) - V_h^{\pi^k}(s_h^k) = Q^{k-1}_h(s^k_h, a^k_h) - Q^{\pi^k}_h(s^k_h, a^k_h) \\
\notag  &\qquad= Q^{k-1}_h(s^k_h, a^k_h) - Q_h^{\star}(s^k_h, a^k_h) + Q_h^{\star}(s^k_h, a^k_h) - Q^{\pi^k}_h(s^k_h, a^k_h) \\
\notag &\qquad= Q^{k-1}_h(s^k_h, a^k_h) - Q_h^{\star}(s^k_h, a^k_h) + P_{h, s^k_h, a^k_h}(V^{\star}_{h+1} - V^{\pi^k}_{h+1}) \\
&\qquad= Q^{k-1}_h(s^k_h, a^k_h) - Q_h^{\star}(s^k_h, a^k_h) + \big(P_{h, s^k_h, a^k_h} - P^k_h\big)(V^{\star}_{h+1} - V^{\pi^k}_{h+1}) + V^{\star}_{h+1}(s^k_{h+1}) - V^{\pi^k}_{h+1}(s^k_{h+1}),
\end{align}
where the third line invokes Bellman equation $Q^{\pi}(s, a) = r(s, a) + P_{h, s, a}V^{\pi}$ for any $\pi$,
and the last line makes use of the notation \eqref{eq:defn-Ph-k-s}. 
Combining the above two properties leads to
\begin{align*}
&\sum_{k \in \mathcal{I}_{1}^K} \big[V_1^{\star}(s_1^k) - V_1^{\pi^k}(s_1^k)\big] \le \sum_{k \in \mathcal{I}_{1}^K} \big[V_1^{k-1}(s_1^k) - V_1^{\pi^k}(s_1^k)\big] \\
&= \sum_{k \in \mathcal{I}_{1}^K}  \big[V_2^{\star}(s_2^k) - V_2^{\pi^k}(s_2^k)\big] + \sum_{k \in \mathcal{I}_{1}^K} \left[Q^{k-1}_1(s^k_1, a^k_1) - Q_1^{\star}(s^k_1, a^k_1) + (P_{1, s^k_1, a^k_1} - P^k_1)(V^{\star}_{2} - V^{\pi^k}_{2})\right] \\
&\le \sum_{k \in \mathcal{I}_{2}^K}  \big[V_2^{\star}(s_2^k) - V_2^{\pi^k}(s_2^k)\big] + \sum_{k \in \mathcal{I}_{1}^K} \left[Q^{k-1}_1(s^k_1, a^k_1) - Q_1^{\star}(s^k_1, a^k_1) + (P_{1, s^k_1, a^k_1} - P^k_1)(V^{\star}_{2} - V^{\pi^k}_{2})\right],
\end{align*}
where the last line comes from the observation that $\mathcal{I}_{h}^K \subseteq \mathcal{I}_{h+1}^K$ (see Lemma~\ref{lem:basic-facts-Ik}).
Applying the above relation recursively and using the fact that $V^{\pi}_{H+1}=0$ for any $\pi$,
we see that with probability at least $1-\delta$,  
\begin{align*}
\sum_{k \in \mathcal{I}_{1}^K} \big[V_1^{\star}(s_1^k) - V_1^{\pi^k}(s_1^k)\big] &\le \sum_{h = 1}^H\sum_{k \in \mathcal{I}_{h}^K} \left[Q^{k-1}_h(s^k_h, a^k_h) - Q_h^{\star}(s^k_h, a^k_h) 
+ (P_{h, s^k_h, a^k_h} - P^k_h)(V^{\star}_{h+1} - V^{\pi^k}_{h+1})\right] \\
&\le \sum_{h = 1}^H\sum_{k \in \mathcal{I}_{h}^K}  \left[ 2b_h^{k-1}(s^k_h, a^k_h) + (P_{h, s^k_h, a^k_h} - P^k_h)(V^{\star}_{h+1} - V^{\pi^k}_{h+1})\right]\\
&= \underbrace{\sum_{h = 1}^H\sum_{k \in \mathcal{I}_{h}^K} 2b_h^{k-1}(s^k_h, a^k_h)}_{\eqqcolon \, \xi_1} + 
\underbrace{\sum_{h = 1}^H\sum_{k \in \mathcal{I}_{h}^K} (P_{h, s^k_h, a^k_h} - P^k_h)(V^{\star}_{h+1} - V^{\pi^k}_{h+1}) }_{\eqqcolon\,\xi_2},
\end{align*}
where the second inequality invokes Lemma~\ref{lem:Qk-error}. 
Therefore, it is sufficient to bound $\xi_1$ and  $\xi_2$ separately, which we accomplish as follows.
\begin{itemize}
\item Regarding the term $\xi_{2}$, we first make the observation that $\big\{(P_{h, s^k_h, a^k_h} - P^k_h)(V^{\star}_{h+1} - V^{\pi^k}_{h+1}) \big\}$ forms a martingale difference sequence, 
as $\pi^k_j$ is determined by $Q_j^{k-1}(s_j^k, a)$. 
Moreover, the sequence satisfies the trivial bound $$\left|(P_{h, s^k_h, a^k_h} - P^k_h)(V^{\star}_{h+1} - V^{\pi^k}_{h+1}) \right| \le H.$$
		These properties allow us to apply the celebrated Azuma-Hoeffding inequality~\citep{azuma1967weighted}, which together with the trivial upper bound $\sum_{h=1}^H|\mathcal{I}_{h}^K|\leq KH$ ensures that  
\begin{align}
 | \xi_2 | = \bigg| \sum_{h = 1}^H\sum_{k \in \mathcal{I}_{h}^K} (P_{h, s^k_h, a^k_h} - P^k_h)(V^{\star}_{h+1} - V^{\pi^k}_{h+1}) \bigg|
%
&\le H \sqrt{H K \log\frac{2}{\delta}}
\label{eqn:delta2}
\end{align}
with probability at least $1-\delta$. 

\item Turning to the term $\xi_{1}$, we apply 
the Cauchy-Schwarz inequality to derive
\begin{align}
	\xi_1 &= \sum_{h = 1}^H\sum_{k \in \mathcal{I}_{h}^K} 2\beta \sqrt{\varphi_h(s^k_h, a^k_h)^{\top}\big(\Lambda_h^{k-1}\big)^{-1}\varphi_h(s^k_h, a^k_h)} \notag\\
	& \le 2\beta \sqrt{KH} \sqrt{\sum_{h = 1}^H\sum_{k \in \mathcal{I}_{h}^K} \varphi_h(s^k_h, a^k_h)^{\top}\big(\Lambda_h^{k-1}\big)^{-1}\varphi_h(s^k_h, a^k_h)}.
	\label{eqn:delta1}
\end{align}
To further control the right-hand side of \eqref{eqn:delta1},  we resort to Lemma~\ref{lem:sum-normalized} in Section~\ref{sec:auxiliary-lemmas} --- a result borrowed from \citep{abbasi2011improved} ---
which immediately leads to 
\begin{align}
\label{eqn:delta1-2}
	\xi_1 \leq 2\beta\sqrt{KH} \cdot \sqrt{2Hd \log (KH)} = 2H\beta\sqrt{2dK\log (KH)}.
\end{align}
\end{itemize}
Putting everything together gives 
\begin{align}
	 \sum_{k \in \mathcal{I}_{1}^K} \big[V_1^{\star}(s_1^k) - V_1^{\pi^k}(s_1^k)\big]
	 &\le 2H\beta\sqrt{2dK\log (KH)} + \sqrt{H^3K \log \frac{2}{\delta}}\le 4c_{\beta}\sqrt{d^2H^6K\log^2 \frac{KH}{\delta}},
	\label{eqn:buenos-aires}
\end{align}
where the last inequality makes use of the definition $\beta \defn c_{\beta} \sqrt{dH^4\log \frac{KH}{\delta}}$.

\paragraph{Step 3: bounding the number of state revisits.}

To this end, we make the observation that $N = \left|\mathcal{I}_{1}^K\right|$ and $K = \left|\mathcal{I}_{H}^K\right|$ (see Lemma~\ref{lem:basic-facts-Ik}).
With this in mind, we can bound the total number $K-N$ of state revisits as follows: 
\begin{align}
	K-N 
	= \left|\mathcal{I}_{H}^K\right| - \left|\mathcal{I}_{1}^K\right| 
	= \sum_{h = 1}^H\big(\left|\mathcal{I}_{h+1}^K\right|-\left|\mathcal{I}_{h}^K\right|\big) 
	\le \frac{4c_{\beta}^2d^2H^5\log^2 \frac{KH}{\delta}}{\taugap^2}. 
	\label{eq:number-leafs}
\end{align} 
Here, the above inequality is a consequence of the auxiliary lemma below, whose proof is provided in Section~\ref{sec:invierno-porteno}.
\begin{lemma}  \label{lem:number-nodes}
Suppose that $KH\geq 2$. For all $1\leq h< H$, the following condition
\begin{align}
	\left|\mathcal{I}_{h+1}^K \setminus \mathcal{I}_{h}^K\right| \le \frac{4c_{\beta}^2d^2H^4\log^2 \frac{KH}{\delta}}{\taugap^2}  \label{eqn:restarts}
\end{align}
holds, where $c_{\beta}>0$ is the pre-constant defined in \eqref{eq:defn-beta}.
\end{lemma}


\paragraph{Step 4: sample complexity analysis.}
Recall that $T$ stands for the total number of samples collected, which clearly satisfies $T\geq K$. 
Consequently, the above results \eqref{eq:number-leafs} and \eqref{eqn:buenos-aires} taken collectively lead to
\begin{align}
	0 \leq K-N \leq \frac{4c_{\beta}^2d^2H^5\log^2 \frac{TH}{\delta}}{\taugap^2} \leq N 
	\qquad \Longrightarrow \qquad N\leq K\leq 2N, 
	\label{eq:K-N-relation}
\end{align}
provided that  $N \geq \frac{4c_{\beta}^{2}d^{2}H^{5}\log^{2}\frac{HT}{\delta}}{\taugap^{2}}$. 
This together with the fact $T\leq KH$ implies that 
\begin{align}
	T \leq KH \leq 2NH 
	\qquad \Longrightarrow \qquad 
	N \geq \frac{T}{2H}. \label{eq:T-upper-bound-KH} 
\end{align}
As a result,  we can invoke \eqref{eqn:buenos-aires} to obtain
\begin{align}
\frac{1}{N}\sum_{n=1}^{N}\Big(V_{1}^{\star}(s_{1}^{(n)})-V_{1}^{\pi^{(n)}}(s_{1}^{(n)})\Big) & =\frac{1}{N}\sum_{k\in\mathcal{I}_{1}^{K}}\big[V_{1}^{\star}(s_{1}^{k})-V_{1}^{\pi^{k}}(s_{1}^{k})\big]\le4c_{\beta}\sqrt{\frac{d^{2}H^{6}K\log^{2}\frac{HT}{\delta}}{N^{2}}}\notag\\
 & \leq 8 c_{\beta}\sqrt{\frac{d^{2}H^{7}\log^{2}\frac{HT}{\delta}}{T}} ,
	\label{eq:regret-average-T}
\end{align}
where the last relation arises from both \eqref{eq:K-N-relation} and \eqref{eq:T-upper-bound-KH}. 
This concludes the proof.

\subsection{Analysis for regret over the paths (proof of \eqref{eq:regret_path})}
\label{sec:analysis-proof-regret-path}

\paragraph{Proof of \eqref{eq:worst_case}.} 

Invoking the crude bound $V^{\star}(s_1^k) - V^{\pi^k}(s_1^k) \le H$ leads to 
\begin{align}
\mathsf{Regret}_{\mathsf{path}} (K) &= \sum_{k\in\mathcal{I}_{1}^{K}}\big[V_{1}^{\star}(s_{1}^{k})-V_{1}^{\pi^{k}}(s_{1}^{k})\big] + \sum_{k\notin\mathcal{I}_{1}^{K}}\big[V_{1}^{\star}(s_{1}^{k})-V_{1}^{\pi^{k}}(s_{1}^{k})\big] \nonumber\\
&\le \sum_{k\in\mathcal{I}_{1}^{K}}\big[V_{1}^{\star}(s_{1}^{k})-V_{1}^{\pi^{k}}(s_{1}^{k})\big] + H \big(K- |\mathcal{I}_{1}^{K}| \big) \nonumber\\
&= \sum_{k\in\mathcal{I}_{1}^{K}}\big[V_{1}^{\star}(s_{1}^{k})-V_{1}^{\pi^{k}}(s_{1}^{k})\big] + H(K-N) \nonumber\\
&\le 4c_{\beta}\sqrt{d^2H^6K\log^2\frac{HT}{\delta}} + \frac{4c_{\beta}^2d^2H^6\log^2 \frac{TH}{\delta}}{\Delta_{\mathsf{gap}}^2}, 
\end{align}
where the penultimate line relies on \eqref{eq:I-end}, 
and the last relation holds due to 	
\eqref{eqn:buenos-aires} and \eqref{eq:number-leafs}.


\paragraph{Proof of \eqref{eq:expected_regret} (logarithmic regret).} As it turns out, this logarithmic regret bound (w.r.t.~$K$) can be established by combining our result with a result derived in~\citet{yang2021q}. 
To be precise, by defining $\Delta_h(s, a) := V_h^{\star}(s) - Q_h^{\star}(s, a)$, we make the following observation: 
$$
\mathsf{Regret}_{\mathsf{path}}(K) = \sum_{k = 1}^K \big(V_1^{\star}(s_1^k) - V_1^{\pi^k}(s_1^k)\big)
	= \sum_{k = 1}^K \mathbb{E}\Bigg[ \sum_{h = 1}^H  \Delta_h(s_h^k, a_h^k) \mid \pi^k, s_1^k \Bigg], 
$$
which has been derived in \citet[Equation (1)]{yang2021q}. Unconditioning gives
\begin{align}
	\mathbb{E} \big[ \mathsf{Regret}_{\mathsf{path}}(K) \big]  
	& =  \mathbb{E} \Bigg[ \sum_{k = 1}^K \big(V_1^{\star}(s_1^k) - V_1^{\pi^k}(s_1^k)\big) \Bigg]
	= \mathbb{E}\Bigg[ \sum_{k = 1}^K  \sum_{h = 1}^H  \Delta_h(s_h^k, a_h^k)  \Bigg]
	= \mathbb{E}\Bigg[ \sum_{h = 1}^H  \sum_{k = 1}^K   \Delta_h(s_h^k, a_h^k)  \Bigg] \nonumber\\
	&= \mathbb{E}\Bigg[ \sum_{h = 1}^H  \sum_{k \in \mathcal{I}_h^K }   \Delta_h(s_h^k, a_h^k)  \Bigg] 
	+  \mathbb{E}\Bigg[ \sum_{h = 1}^H  \sum_{k \notin \mathcal{I}_h^K}   \Delta_h(s_h^k, a_h^k)  \Bigg]. 
	\label{eq:regret-expansion-123}
\end{align}

In addition, we make note of the fact that: with probability at least $1-\delta$, one has 
\[
	\Delta_h(s_h^k, a_h^k) = 0 \qquad \text{for all} \quad k \in \mathcal{I}_h^K,
\]
which follows immediately from the update rule of Algorithm~\ref{alg:linearQ} (cf.~line~\ref{line:select-b}) and Lemma~\ref{lem:Qk-error}. 
This taken collectively with the trivial bound $\Delta_h(s_h^k,a_h^k) \leq H$ gives
\[
	\mathbb{E}\Bigg[ \sum_{h = 1}^H  \sum_{k \in \mathcal{I}_h^K }   \Delta_h(s_h^k, a_h^k)  \Bigg]
	\leq (1-\delta)\cdot 0 + \delta \cdot  \sum_{h = 1}^H  \sum_{k \in \mathcal{I}_h^K }  H \leq H^2K\delta. 
\] 
%
%
Substitution into \eqref{eq:regret-expansion-123} yields
\begin{align*}
 \mathbb{E} \big[ \mathsf{Regret}_{\mathsf{path}}(K) \big] 
	&\leq H^2K\delta + \sum_{h=1}^{H}\sum_{k\notin\mathcal{I}_{h}^{K}}\mathbb{E}\Big[\Delta_{h}(s_{h}^{k},a_{h}^{k})\Big]
	\leq  H^2K\delta + \sum_{h=1}^{H}\sum_{k\notin\mathcal{I}_{h}^{K}} H \\
	& \le H^2K\delta + H^{2}(K-N)
	\le H^2K\delta + \frac{4c_{\beta}^{2}d^{2}H^{7}\log^{2}\frac{TH}{\delta}}{\Delta_{\mathsf{gap}}^{2}}. 
\end{align*}
%
Here, the second inequality follows from the trivial upper bound $\Delta_h(s_h^k, a_h^k) \le H$, 
the third inequality holds true since $K - \big|\mathcal{I}_h^K\big| \le K - \big|\mathcal{I}_H^K\big| = K - N$ (see Lemma~\ref{lem:basic-facts-Ik}), 
whereas the last inequality is valid due to \eqref{eq:number-leafs}. 
Taking $\delta = 1 / K$ and recalling that $T\leq KH$, we arrive at the advertised logarithmic regret bound: 
\[
 	\mathbb{E} \big[ \mathsf{Regret}_{\mathsf{path}}(K) \big] 
	\le \frac{17c_{\beta}^{2}d^{2}H^{7}\log^{2}(KH)}{\Delta_{\mathsf{gap}}^{2}}.
\]

%

%

%


%% file: conclusion.tex
\section{Discussion}
\label{sec:discussion}

In this paper, we have made progress towards understanding the plausibility of achieving sample-efficient RL when the optimal Q-function is linearly realizable.  While prior works suggested an exponential sample size barrier in the standard online RL setting even in the presence of a constant sub-optimality gap, 
we demonstrate that this barrier can be conquered by permitting state revisiting (also called local access to generative models). 
An algorithm called LinQ-LSVI-UCB has been developed that provably enjoys a reduced sample complexity, which is polynomial in the feature dimension, the horizon and the inverse sub-optimality gap, but otherwise independent of the dimension of the state/action space.

Note, however, that linear function approximation for online RL remains a rich territory for further investigation. 
In contrast to the tabular setting,  the feasibility and limitations of online RL might vary drastically across different families of linear function approximation. 
There are numerous directions that call for further theoretical development in order to obtain a more complete picture. 
For instance, can we identify other flexible, yet practically relevant, online RL sampling mechanisms that also allow for sample size reduction? 
Can we derive the information-theoretic sampling limits for various linear function approximation classes, 
and characterize the fundamental interplay between low-dimensional representation and sampling constraints?
Moving beyond linear realizability assumptions, a very recent work \citet{yin2021efficient} showed that a gap-independent sample size reduction is feasible by assuming that $Q^{\pi}$ is linearly realizable for any policy $\pi$. 
However, 
what is the sample complexity limit for this class of function approximation remains largely unclear, particularly when 
state revisiting is not permitted.  All of these are interesting questions for  future studies.


%% file: proof-main-lemmas.tex
\section{Proof of technical lemmas}

\subsection{Proof of Lemma~\ref{lem:Qk-error}}
\label{sec:pf-lem-qk}

\paragraph{Step 1: decomposition of $\theta_h^k - \theta_h^{\star}$. }

To begin with, recalling the update rule~\eqref{eqn:update-theta}, we have the following decomposition 
\begin{align}
\label{eqn:decomposition}
\theta_h^k - \theta_h^{\star} &= \big(\Lambda_h^{k}\big)^{-1} \bigg\{ \sum_{i \in \mathcal{I}_{h}^k} \varphi_h(s_h^i, a_h^i)\big[r_h(s_h^i, a_h^i) + 
	\big\langle \varphi_{h+1}(s_{h+1}^i, a_{h+1}^i), \theta_{h+1}^{k} \big\rangle \big] - \Lambda_h^{k}\theta_h^{\star} \bigg\} \nonumber\\
&= \big(\Lambda_h^{k}\big)^{-1} \bigg\{\sum_{i \in \mathcal{I}_{h}^k} \varphi_h(s_h^i, a_h^i)\big[ \big\langle \varphi_{h+1}(s_{h+1}^i, a_{h+1}^i), \theta_{h+1}^{k} \big\rangle - P_{h, s_h^i, a_h^i}V_{h+1}^{\star}\big] -  \theta_h^{\star} \bigg\}.
\end{align}
To see why the second identity holds, note that from the definition of $\Lambda_h^{k}$ (cf.~\eqref{eqn:def-Lambda-hk}) we have
\begin{align*}
	\Lambda_h^{k}\theta_h^{\star} &= \sum_{i \in \mathcal{I}_{h}^k} \varphi_h(s_h^i, a_h^i) \big( \varphi_h(s_h^i, a_h^i) \big)^{\top} \theta_h^{\star} +   \theta_h^{\star} \\
&= \sum_{i \in \mathcal{I}_{h}^k} \varphi_h(s_h^i, a_h^i)Q_h^{\star}(s_h^i, a_h^i) +   \theta_h^{\star} \\
&= \sum_{i \in \mathcal{I}_{h}^k} \varphi_h(s_h^i, a_h^i)\big[r_h(s_h^i, a_h^i) + P_{h, s_h^i, a_h^i}V_{h+1}^{\star}\big] +   \theta_h^{\star},
\end{align*}
where the second and the third identities invoke the linear realizability assumption of $Q_h^{\star}$ and the Bellman equation, respectively.

As a result of \eqref{eqn:decomposition}, to control the difference $\theta_h^k - \theta_h^{\star}$, it is sufficient to bound 
$\big\langle \varphi_{h+1}(s_{h+1}^i, a_{h+1}^i), \theta_{h+1}^k \big\rangle - P_{h, s_h^i, a_h^i}V_{h+1}^{\star}$. 
Towards this, we start with the following decomposition 
\begin{align*}
 \big\langle \varphi_{h+1}(s_{h+1}^i, a_{h+1}^i), \theta_{h+1}^k \big\rangle - P_{h, s_h^i, a_h^i}V_{h+1}^{\star}& = \big\langle \varphi_{h+1}(s_{h+1}^i, a_{h+1}^i), \theta_{h+1}^{k} \big\rangle - Q_{h+1}^{\star}(s_{h+1}^i, a_{h+1}^i) \\
&\; + Q_{h+1}^{\star}(s_{h+1}^i, a_{h+1}^i) - V_{h+1}^{\star}(s_{h+1}^i) + V_{h+1}^{\star}(s_{h+1}^i) - P_{h, s_h^i, a_h^i}V_{h+1}^{\star}.
\end{align*}
For notational simplicity, let us define
\begin{subequations}
\begin{align}
	\varepsilon_h^k &\defn \left[ \big\langle \varphi_h(s_h^i, a_h^i), \theta_h^k \big\rangle - Q_h^{\star}(s_h^i, a_h^i)\right]_{i \in \mathcal{I}_{h}^k} &&  \in \mathbb{R}^{|  \mathcal{I}_{h}^k |}, \label{eq:defn-epsilon-k}\\
	\delta_h^k &\defn \left[Q_{h+1}^{\star}(s_{h+1}^i, a_{h+1}^i) - V_{h+1}^{\star}(s_{h+1}^i)\right]_{i \in \mathcal{I}_{h}^k} &&  \in \mathbb{R}^{|  \mathcal{I}_{h}^k |}, \\
	\xi_h^k &\defn \left[V_{h+1}^{\star}(s_{h+1}^i) - P_{h, s_h^i, a_h^i}V_{h+1}^{\star}\right]_{i \in \mathcal{I}_{h}^k} & &\in \mathbb{R}^{|  \mathcal{I}_{h}^k |}, \\
	\Phi_h^k &\defn \left[\varphi_h(s_h^i, a_h^i)\right]_{i \in \mathcal{I}_{h}^k} && \in \mathbb{R}^{ d\times |  \mathcal{I}_{h}^k |} .
\end{align}
\end{subequations}
Here and throughout, for any $z= [z_i]_{1\leq i\leq K}$, the vector $[z_i]_{ i\in \mathcal{I}_{h}^k}$ denotes a subvector of $z$ formed by the entries with indices coming from $\mathcal{I}_{h}^k$;
for any set of vectors $w_1,\cdots, w_K$, the matrix $[w_i]_{ i\in \mathcal{I}_{h}^k}$ represents a submatrix of $[w_1,\cdots,w_K]$ whose columns are formed by the vectors with indices coming from $\mathcal{I}_{h}^k$. 
Armed with this set of notation, $\theta_h^k - \theta_h^{\star}$ can be succinctly expressed as 
\begin{align} 
\label{eqn:decomposition-new}
\theta_h^k - \theta_h^{\star} = 
\big(\Lambda_h^{k}\big)^{-1} \Big\{\Phi_h^k \left(\left[\varepsilon_{h+1}^k\right]_{i \in \mathcal{I}_{h}^k} + \delta_{h}^k + \xi_h^k\right) - \theta_h^{\star}\Big\},
\end{align}
where we further define
\begin{align*}
	\left[\varepsilon_{h+1}^k\right]_{i \in \mathcal{I}_{h}^k} \defn	\Big[ \big\langle \varphi_{h+1}(s_{h+1}^i, a_{h+1}^i), \theta_{h+1}^k \big\rangle - Q_{h+1}^{\star}(s_{h+1}^i, a_{h+1}^i)\Big]_{i \in \mathcal{I}_{h+1}^k \cap \mathcal{I}_{h}^k};
\end{align*}
in other words, we consider the vector $\varepsilon_{h+1}^k$ when restricted to the index set $\mathcal{I}_{h+1}^k \cap \mathcal{I}_{h}^k$. 
Recognizing that $\mathcal{I}_{h}^k \subseteq \mathcal{I}_{h+1}^k$ (see Lemma~\ref{lem:basic-facts-Ik}), we can also simply write
\begin{align*}
	\left[\varepsilon_{h+1}^k\right]_{i \in \mathcal{I}_{h}^k} =  
	\Big[ \big\langle \varphi_{h+1}(s_{h+1}^i, a_{h+1}^i), \theta_{h+1}^k \big\rangle - Q_{h+1}^{\star}(s_{h+1}^i, a_{h+1}^i)\Big]_{i \in \mathcal{I}_{h}^k }.
\end{align*}

\paragraph{Step 2: decomposition of $Q_h^k(s, a) - Q_h^{\star}(s, a)$. }

We now employ the above decomposition of  $\theta_h^k - \theta_h^{\star}$ to help control  $Q_h^k(s, a) - Q_h^{\star}(s, a)$
--- the target quantity of  Lemma~\ref{lem:Qk-error}.
By virtue of the relation~\eqref{eqn:decomposition-new}, our estimate $\big\langle\varphi_{h}(s,a),\theta_{h}^{k}\big\rangle$ of the linear representation  $Q_{h}^{\star}(s,a)=\big\langle\varphi_{h}(s,a),\theta_{h}^{\star}\big\rangle$
satisfies 
\begin{align*}
 & \big|\big\langle\varphi_{h}(s,a),\theta_{h}^{k}\big\rangle-Q_{h}^{\star}(s,a)\big|=\big|\big\langle\varphi_{h}(s,a),\theta_{h}^{k}-\theta_{h}^{\star}\big\rangle\big|\nonumber\\
 & \qquad=\left|\varphi_{h}(s,a)^{\top}\big(\Lambda_{h}^{k}\big)^{-1}\Big\{\Phi_{h}^{k}\big(\left[\varepsilon_{h+1}^{k}\right]_{i\in\mathcal{I}_{h}^{k}}+\delta_{h}^{k}+\xi_{h}^{k}\big)-\theta_{h}^{\star}\Big\}\right|\nonumber\\
 & \qquad\le\left\Vert \big(\Lambda_{h}^{k}\big)^{-1/2}\varphi_{h}(s,a)\right\Vert _{2}\cdot\nonumber\\
 & \qquad\qquad\left(\left\Vert \big(\Lambda_{h}^{k}\big)^{-1/2}\Phi_{h}^{k}\right\Vert \left\{ \left\Vert \varepsilon_{h+1}^{k}\right\Vert _{2}+\left\Vert \delta_{h}^{k}\right\Vert _{2}\right\} +\left\Vert \big(\Lambda_{h}^{k}\big)^{-1/2}\Phi_{h}^{k}\xi_{h}^{k}\right\Vert _{2}+\left\Vert \big(\Lambda_{h}^{k}\big)^{-1/2}\right\Vert \left\Vert \theta_{h}^{\star}\right\Vert _{2}\right),\nonumber
\end{align*}
where $\|M\|$ denotes the spectral norm of a matrix $M$. Here, the last inequality follows from the Cauchy-Schwarz inequality and the triangle inequality.
Now, from the definition 
\begin{align}
	\Lambda_h^{k} &= \sum_{i \in \mathcal{I}_{h}^k} \varphi_h(s_h^i, a_h^i)\varphi_h(s_h^i, a_h^i)^{\top} +  I = \Phi_h^k (\Phi_h^k)^\top  +  I,
\end{align}
it is easily seen that $\big\|\big(\Lambda_h^{k}\big)^{-1/2} \big\| \le 1$ and $\big\|\big(\Lambda_h^{k}\big)^{-1/2}\Phi_h^k \big\| \le 1$. 
Consequently, it is guaranteed that 
\begin{align}
	& \big| \big\langle \varphi_h(s, a), \theta_h^k \big\rangle - Q_{h}^{\star}(s, a)\big| \nonumber\\
	& \qquad \le \left( \big\|\varepsilon_{h+1}^k \big\|_2 + \big\|\delta_{h}^k \big\|_2 + \big\|\big(\Lambda_h^{k}\big)^{-1/2}\Phi_h^k\xi_h^k \big\|_2 + \left\|\theta_h^{\star}\right\|_2\right) 
	\cdot \left\|\big(\Lambda_h^{k}\big)^{-1/2}\varphi_h(s, a) \right\|_2.
	\label{eqn:eps-decomposition}
\end{align}

In the sequel, we seek to establish, by induction, that 
\begin{align} \label{eq:error-Qk}
\big| \big\langle \varphi_h(s, a), \theta_h^k \big\rangle - Q_{h}^{\star}(s, a)\big| \le b_h^{k}(s, a) .
\end{align}
If this condition were true, then combining this with the definition (see \eqref{eqn:update-Q})
\[
	Q_h^{k}(s, a) =  \min\Big\{ \big\langle \varphi_h(s, a), \theta_h^{k} \big\rangle + b_h^{k}(s, a), \, H \Big\}
\]
and  the constraint $Q_h^{\star}(s, a)\leq H$ would immediately lead to
\begin{align*}
Q_{h}^{k}(s,a)-Q_{h}^{\star}(s,a) & \leq\big|\big\langle\varphi_{h}(s,a),\theta_{h}^{k}\big\rangle+b_{h}^{k}(s,a)-Q_{h}^{\star}(s,a)\big|\leq2b_{h}^{k}(s,a),\\
Q_{h}^{k}(s,a)-Q_{h}^{\star}(s,a) & \geq\begin{cases}
0, & \text{if }Q_{h}^{k}(s,a)\geq H,\\
b_{h}^{k}(s,a)-\big|\big\langle\varphi_{h}(s,a),\theta_{h}^{k}\big\rangle-Q_{h}^{\star}(s,a)\big|\geq0, & \text{else},
\end{cases}
\end{align*}
as claimed in the inequality \eqref{eq:Qk-error} of this lemma.  Consequently, everything boils down to establishing \eqref{eq:error-Qk}, which forms the main content of the next setp.

\paragraph{Step 3: proof of the inequality  \eqref{eq:error-Qk}.}

The proof of this inequality proceeds by bounding each term in the relation \eqref{eqn:eps-decomposition}.

To start with, we establish an upper bound on $\left\|\big(\Lambda_h^{k}\big)^{-1/2}\Phi_h^k\xi_h^k\right\|_2$ 
as required in our advertised inequality~\eqref{eq:error-Qk}; that is, with probability at least $1-\delta,$  
\begin{align}
\label{eqn:xi-bound}
	\left\|\big(\Lambda_j^{k}\big)^{-1/2}\Phi_j^k\xi_j^k\right\|_2
	\leq \sqrt{H^2 \left(d\log (KH) +2\log\frac{KH}{\delta} \right)} 
	\leq 2 \sqrt{H^2 d\log \frac{KH}{\delta} }
\end{align}
holds simultaneously for all $1\leq j \leq H$ and $1\leq k \leq K$. 
Towards this end, let us define 
\begin{align}
	X_i \defn V_{j+1}^{\star}(s_{j+1}^i) - P_{j, s_j^i, a_j^i}V_{j+1}^{\star}.
\end{align}
It is easily seen that $\{X_{i}\}_{i\in \mathcal{I}_{j}^K}$ forms a martingale sequence.
%
%
In addition, we have the trivial upper bound $|X_i| \le H$. 
Therefore, applying the concentration inequality for self-normalized processes (see Lemma~\ref{lem:sum-normalized} in Section~\ref{sec:auxiliary-lemmas}), 
we can deduce that 
\begin{align}
	\left\|\big(\Lambda_j^{k}\big)^{-1/2}\Phi_j^k\xi_j^k\right\|_2 & = 
	\Big\|\big(\Lambda_j^{k}\big)^{-1/2} \sum_{i \in \mathcal{I}_{j}^k} \varphi_j(s_j^i, a_j^i) X_i\Big\|_2 
\le 
\sqrt{H^2\log\left(\frac{\mathrm{det}\left(\Lambda_j^{k}\right)}{\mathrm{det}\left(\Lambda_j^{0}\right)\delta^2}\right)}  \nonumber\\
	&\le \sqrt{H^2 \left(d\log (KH) +2\log\frac{1}{\delta} \right)} 
\label{eqn:xi-bound-pre}
\end{align}
holds with probability at least $1-\delta$.
Here, the first inequality comes from  Lemma~\ref{lem:normalized} in Section~\ref{sec:auxiliary-lemmas},
whereas the second inequality is a consequence of Lemma~\ref{lem:sum-normalized} in Section~\ref{sec:auxiliary-lemmas}.
 Taking the union bound  over $j=1,\ldots, H$ and $k=1,\ldots, K$ yields the required relation~\eqref{eqn:xi-bound}.

Armed with the above inequality, we can move on to establish the inequality~\eqref{eq:error-Qk} by induction, working backwards. 
First, we observe that when $h = H+1$, the inequality~\eqref{eq:error-Qk} holds true trivially (due to the initialization $\theta_{H+1}^k=0$ and the fact $Q_{H+1}^{\star}=0$). 
Next, let us assume that the claim holds for $h+1, \ldots, H+1$, and show that the claim continues to hold for step $h$.
To this end,  it is sufficient to bound the terms in \eqref{eqn:eps-decomposition} separately.

\begin{itemize}

\item {\bf The term $\|\delta_{j}^k\|_2$.}
From the induction hypothesis, the inequality~\eqref{eq:error-Qk} holds for $h+1,\ldots, H+1$.
For any $j$ obeying $h \leq j \leq H$ and any $i\in \mathcal{I}_j^k$, 
this in turn guarantees that
\begin{align*}
	V_{j+1}^{\star}(s_{j+1}^i) & = Q_{j+1}^\star \big( s_{j+1}^i, \pi^\star(s_{j+1}^i) \big) 
	 \leq Q_{j+1}^{i-1} \big( s_{j+1}^i, \pi^\star(s_{j+1}^i) \big) \leq Q_{j+1}^{i-1}(s_{j+1}^i, a_{j+1}^i) \\
	& \leq \big\langle  \varphi_h(s_{j+1}^i, a_{j+1}^i), \theta_{j+1}^{i-1} \big\rangle  + b_{j+1}^{i-1}(s_{j+1}^i, a_{j+1}^i) . 
\end{align*}
%
Here, the second inequality follows since $a_{j+1}^i$ is chosen to be an action maximizing $Q_{j+1}^{i-1}(s_{j+1}^i, \cdot) $. 
In the meantime, the induction hypothesis \eqref{eq:error-Qk} for $h+1,\ldots, H+1$ also implies
\begin{align*}
 Q_{j+1}^{\star}(s_{j+1}^i, a_{j+1}^i) 
	&\ge \big\langle \varphi_h(s_{j+1}^i, a_{j+1}^i), \theta_{j+1}^{i-1} \big\rangle - b_{j+1}^{i-1}(s_{j+1}^i, a_{j+1}^i)
\end{align*}
for all $j$ obeying $h \leq j \leq H$. 
Taken collectively, the above two inequalities demonstrate that
\begin{align*}
0 \le V_{j+1}^{\star}(s_{j+1}^i) - Q_{j+1}^{\star}(s_{j+1}^i, a_{j+1}^i) \le 2b_{j+1}^{i-1}(s_{j+1}^i, a_{j+1}^i),
\qquad
i \in \mathcal{I}_{j}^k,
\end{align*}
where the first inequality holds trivially since $V_{j+1}^{\star}(s)=\max_a Q_{j+1}^{\star}(s,a)$. 
Then, given that $i \in \mathcal{I}_{j}^k$, 
one necessarily has $b_{j+1}^{i-1}(s_{j+1}^i, a_{j+1}^i) < \taugap/2$, 
which combined with the sub-optimality gap assumption \eqref{eq:defn-gap-h} implies that $a_{j+1}^i$ cannot be a sub-optimal action in state $s_{j+1}^i$.
Consequently, we reach
\[
V_{j+1}^{\star}(s_{j+1}^i) - Q_{j+1}^{\star}(s_{j+1}^i, a_{j+1}^i) = 0\qquad\text{for all }i \in \mathcal{I}_{j}^k,
\]
and as a result, 
\begin{align}
 \delta_{j}^k = 0. \label{eqn:delta-bound}
\end{align}

\item {\bf The term $\|\varepsilon_{h+1}^k\|_2$.}
	Recall that the decomposition~\eqref{eqn:decomposition-new} together with the assumption $Q_h^{\star}(s,a)=\langle \varphi_h(s,a), \theta_h^{\star} \rangle$ allows us to write $\varepsilon_h^k$ (cf.~\eqref{eq:defn-epsilon-k}) as follows
\begin{align*}
\varepsilon_h^k = \left(\Phi_h^k\right)^{\top}\left(\theta_h^k - \theta_h^{\star}\right) = \left(\Phi_h^k\right)^{\top}\big(\Lambda_h^{k}\big)^{-1}\Big\{\Phi_h^k\big(\left[\varepsilon_{h+1}^k\right]_{i \in \mathcal{I}_{h}^k}+\delta_{h}^k+\xi_h^k\big) -  \theta_h^{\star}\Big\}.
\end{align*}
Combining this with the basic properties $\|(\Phi_h^k)^{\top}\big(\Lambda_h^{k}\big)^{-1/2}\| \le 1$, $\|\big(\Lambda_h^{k}\big)^{-1/2}\| \le 1$, and $\mathcal{I}_{h}^K \subseteq \mathcal{I}_{h+1}^K$ yields
\begin{align}
\|\varepsilon_h^k\|_2 &= \left\|\left(\Phi_h^k\right)^{\top}\big(\Lambda_h^{k}\big)^{-1}\Big\{\Phi_h^k\big(\left[\varepsilon_{h+1}^k\right]_{i \in \mathcal{I}_{h}^k}+\delta_{h}^k+\xi_h^k\big) -  \theta_h^{\star}\Big\}\right\|_2 \nonumber\\
&\le \left\|\varepsilon_{h+1}^k\right\|_2+\left\|\delta_{h}^k\right\|_2+ \left\|\big(\Lambda_h^{k}\big)^{-1/2}\Phi_h^k\xi_h^k \right\|_2 +  \left\|\theta_h^{\star}\right\|_2.
\end{align}
Applying this inequality recursively leads to 
\begin{align}
\|\varepsilon_h^k\|_2 &\le \left\|\varepsilon_{H+1}^k\right\|_2 + \sum_{h \le j \le H} \Big[\|\delta_{j}^k\|_2+ \left\|\big(\Lambda_h^{k}\big)^{-1/2}\Phi_j^k\xi_j^k \right\|_2 +  \|\theta_j^{\star}\|_2 \Big] \nonumber\\
&\le 4\sqrt{dH^4\log \frac{KH}{\delta}}, \label{eqn:eps-bound}
\end{align}
where the last inequality holds by putting together the property $\varepsilon_{H+1}^k = 0$, the inequalities~\eqref{eqn:xi-bound} and \eqref{eqn:delta-bound}, and the assumption that $\|\theta_j^{\star}\|_2 \leq 2H\sqrt{d}$ (see \eqref{eq:assumptions-phi-theta-size}). 

\end{itemize}

Combining the inequalities~\eqref{eqn:xi-bound}, \eqref{eqn:delta-bound}, \eqref{eqn:eps-bound} with the relation~\eqref{eqn:eps-decomposition}, we arrive at 
\begin{align*}
	\big| \big \langle \varphi_h(s, a), \theta_h^k \big\rangle - Q_{h}^{\star}(s, a)\big|
	%
	\notag &\le 
	\left(4\sqrt{dH^4\log \frac{KH}{\delta}}
	+ 2 \sqrt{H^2 d\log \frac{KH}{\delta}} + 2H\sqrt{d} \right) \big\|\big(\Lambda_h^{k}\big)^{-1/2}\varphi_h(s, a) \big\|_2\\
	&\le c_{\beta} \sqrt{dH^4\log \frac{KH}{\delta}} \big\|\big(\Lambda_h^{k}\big)^{-1/2}\varphi_h(s, a) \big\|_2 = b_h^k(s,a) ,
\end{align*}
provided that $c_{\beta}\geq 8$. 
This completes the induction step of \eqref{eq:error-Qk},  thus concluding the proof of Lemma~\ref{lem:Qk-error}.

\subsection{Proof of Lemma~\ref{lem:number-nodes}}
\label{sec:invierno-porteno}

Consider any $1\leq h<H$. 
In view of the definition \eqref{eqn:def-b-hk} of the UCB bonus $b_h^k$, one obtains
\begin{align}
	\sum_{k: k+1 \in \mathcal{I}_{h+1}^K} \left[ b_{h+1}^{k-1}(s_{h+1}^k, a_{h+1}^k) \right]^2 
	&\le \sum_{k: k+1  \in \mathcal{I}_{h+1}^K} c_{\beta}^2dH^4
	\left[ \varphi_{h+1}(s_{h+1}^k, a_{h+1}^k) \right]^{\top}\big(\Lambda_{h+1}^{k-1}\big)^{-1}\varphi_{h+1}(s_{h+1}^k, a_{h+1}^k)\log \frac{KH}{\delta} \notag\\
	&\le c_{\beta}^2d^2H^4\log \frac{KH}{\delta}  \log \frac{K+d}{d}
	\le c_{\beta}^2d^2H^4\log^2 \frac{KH}{\delta} \eqqcolon B, 
	\label{eq:sum-bk-square-UB1}
\end{align}
where the second inequality comes from a standard result stated in Lemma~\ref{lem:sum-normalized} of Section~\ref{sec:auxiliary-lemmas},
and the last inequality is valid as long as $KH\geq 2$.

As a direct consequence of the upper bound \eqref{eq:sum-bk-square-UB1}, there are no more than  $\frac{B}{(\taugap/2)^2}$ elements in the set $\mathcal{I}_{h+1}^K$ obeying $b_{h+1}^{k-1}(s_{h+1}^k, a_{h+1}^k) \geq \taugap/2$.
This combined with the update rule in Line~\ref{line:select-b} of Algorithm~\ref{alg:linearQ} immediately yields
\begin{align*}
\left|\mathcal{I}_{h+1}^K \setminus \mathcal{I}_{h}^K\right| \le  \frac{B}{(\taugap/2)^2} = \frac{4c_{\beta}^2d^2H^4\log^2 \frac{KH}{\delta}}{\taugap^2}, 
\end{align*}
since, by construction, any $k$ with $k+1$ belonging to $\mathcal{I}_{h+1}^K \setminus \mathcal{I}_{h}^K$ must violate the condition $b_{h+1}^{k-1}(s_{h+1}^k, a_{h+1}^k) < \taugap/2$. 
This completes the proof.

%% file: auxiliary_lemmas.tex
\subsection{Auxiliary lemmas}
\label{sec:auxiliary-lemmas}

In this section, we provide a couple of auxiliary lemmas that have been frequently invoked in the literature of linear bandits and linear MDPs. 
The first result is concerned with the interplay between the feature map and the (regularized) covariance matrix $\Lambda_h^k$. 
\begin{lemma} [\citet{abbasi2011improved}]
\label{lem:sum-normalized} 
Under the assumption~\eqref{eq:assumptions-phi-theta-size} and the
definition~\eqref{eqn:def-Lambda-hk}, the following relationship holds true
\begin{equation}
\sum_{i \in \mathcal{I}_{h}^k} \varphi_h(s_{h}^i, a_{h}^i)^{\top}\big(\Lambda_h^{i-1}\big)^{-1}\varphi_h(s_{h}^i, a_{h}^i) 
\le 
2\log \left(\frac{\mathrm{det}\left(\Lambda_h^{k}\right)}{\mathrm{det}\left(\Lambda_h^{0}\right)}\right) \le 2d\log\left(\frac{k}{d}+1\right).
\end{equation}
\end{lemma}
\begin{proof}
	The first inequality is an immediate consequence of \citet{abbasi2011improved} or \citet[Lemma D.2]{jin2020provably}. 
	Regarding the second inequality, let $\lambda_i\geq 1$ be the $i$-th largest eigenvalue of the positive-semidefinite matrix $\Lambda_h^{k}$. 
	From the AM-GM inequality, it is seen that
	\begin{align}
		\mathrm{det}\left(\Lambda_h^{k}\right) = \prod_{i=1}^d \lambda_i \leq \left( \frac{\sum_{i=1}^d \lambda_i }{d} \right)^d 
		\leq \left( \frac{k+d}{d} \right)^d,
	\end{align}
	where the last inequality arises since (in view of the assumption \eqref{eq:assumptions-phi-theta-size})
	\[
		\sum_{i=1}^d \lambda_i = \mathsf{Tr}\left(\Lambda_h^{k}\right) = d + \sum_{i\in \mathcal{I}_{h}^k} \big\| \varphi_h(s_{h}^i, a_{h}^i) \big\|_2^2 \leq d+ k. 
	\]
\end{proof}

The second result delivers a concentration inequality for the so-called self-normalized processes. 
\begin{lemma}[\citet{abbasi2011improved}] \label{lem:normalized}
	Assume that $\left\{X_t \in \mathbb{R}\right\}_{t = 1}^{\infty}$ is a martingale w.r.t.~the filtration $\left\{\mathcal{F}_t\right\}_{t = 0}^{\infty}$ obeying
\begin{align*}
\mathbb{E}\left[X_t \mymid \mathcal{F}_{t-1}\right] = 0, \qquad\text{and}\qquad \mathbb{E}\left[e^{\lambda X_t} \mymid \mathcal{F}_{t-1}\right] \le e^{\frac{\lambda^2\sigma^2}{2}},~\forall~\lambda \in \mathbb{R}.
\end{align*}
In addition, suppose that $\varphi_t \in \mathbb{R}^d$ is a random vector over $\mathcal{F}_{t-1}$, and define $\Lambda_t = \Lambda_0 + \sum_{i = 1}^t \varphi_i\varphi_i^{\top} \in \mathbb{R}^{d\times d}$.
Then with probability at least $1-\delta$, it follows that
\begin{align*}
	\Big\|\left(\Lambda_t \right)^{-1/2} \sum_{i = 1}^t \varphi_iX_i\Big\|_2^2 \le \sigma^2\log\left(\frac{\mathrm{det}(\Lambda_t)}{\mathrm{det}(\Lambda_0)\delta^2}\right)  \qquad \text{for all }t\geq 0.
\end{align*}
\end{lemma}

%% file: LinearQ_Revisit.bbl
\begin{thebibliography}{}

\bibitem[Abbasi-Yadkori et~al., 2011]{abbasi2011improved}
Abbasi-Yadkori, Y., P{\'a}l, D., and Szepesv{\'a}ri, C. (2011).
\newblock Improved algorithms for linear stochastic bandits.
\newblock In {\em NIPS}, volume~11, pages 2312--2320.

\bibitem[Agarwal et~al., 2019]{agarwal2019reinforcement}
Agarwal, A., Jiang, N., Kakade, S.~M., and Sun, W. (2019).
\newblock Reinforcement learning: Theory and algorithms.

\bibitem[Agarwal et~al., 2020a]{agarwal2020flambe}
Agarwal, A., Kakade, S., Krishnamurthy, A., and Sun, W. (2020a).
\newblock {FLAMBE}: Structural complexity and representation learning of low
  rank {MDP}s.
\newblock {\em arXiv preprint arXiv:2006.10814}.

\bibitem[Agarwal et~al., 2020b]{agarwal2019optimality}
Agarwal, A., Kakade, S., and Yang, L.~F. (2020b).
\newblock Model-based reinforcement learning with a generative model is minimax
  optimal.
\newblock {\em Conference on Learning Theory}, pages 67--83.

\bibitem[Agarwal et~al., 2020c]{agarwal2020optimality}
Agarwal, A., Kakade, S.~M., Lee, J.~D., and Mahajan, G. (2020c).
\newblock Optimality and approximation with policy gradient methods in {M}arkov
  decision processes.
\newblock In {\em Conference on Learning Theory}, pages 64--66. PMLR.

\bibitem[Ayoub et~al., 2020]{ayoub2020model}
Ayoub, A., Jia, Z., Szepesvari, C., Wang, M., and Yang, L. (2020).
\newblock Model-based reinforcement learning with value-targeted regression.
\newblock In {\em International Conference on Machine Learning}, pages
  463--474. PMLR.

\bibitem[Azar et~al., 2013]{azar2013minimax}
Azar, M.~G., Munos, R., and Kappen, H.~J. (2013).
\newblock Minimax {PAC} bounds on the sample complexity of reinforcement
  learning with a generative model.
\newblock {\em Machine learning}, 91(3):325--349.

\bibitem[Azar et~al., 2017]{azar2017minimax}
Azar, M.~G., Osband, I., and Munos, R. (2017).
\newblock Minimax regret bounds for reinforcement learning.
\newblock In {\em International Conference on Machine Learning}, pages
  263--272.

\bibitem[Azuma, 1967]{azuma1967weighted}
Azuma, K. (1967).
\newblock Weighted sums of certain dependent random variables.
\newblock {\em Tohoku Mathematical Journal, Second Series}, 19(3):357--367.

\bibitem[Baird, 1995]{baird1995residual}
Baird, L. (1995).
\newblock Residual algorithms: Reinforcement learning with function
  approximation.
\newblock In {\em Machine Learning Proceedings 1995}, pages 30--37. Elsevier.

\bibitem[Bertsekas and Tsitsiklis, 1995]{bertsekas1995neuro}
Bertsekas, D.~P. and Tsitsiklis, J.~N. (1995).
\newblock Neuro-dynamic programming: an overview.
\newblock In {\em Proceedings of 1995 34th IEEE conference on decision and
  control}, volume~1, pages 560--564. IEEE.

\bibitem[Cen et~al., 2020]{cen2020fast}
Cen, S., Cheng, C., Chen, Y., Wei, Y., and Chi, Y. (2020).
\newblock Fast global convergence of natural policy gradient methods with
  entropy regularization.
\newblock {\em arXiv preprint arXiv:2007.06558}.

\bibitem[Dann and Brunskill, 2015]{dann2015sample}
Dann, C. and Brunskill, E. (2015).
\newblock Sample complexity of episodic fixed-horizon reinforcement learning.
\newblock In {\em Advances in Neural Information Processing Systems}, pages
  2818--2826.

\bibitem[Dimakopoulou et~al., 2019]{dimakopoulou2019balanced}
Dimakopoulou, M., Zhou, Z., Athey, S., and Imbens, G. (2019).
\newblock Balanced linear contextual bandits.
\newblock In {\em AAAI Conference on Artificial Intelligence}, volume~33, pages
  3445--3453.

\bibitem[Du et~al., 2021]{du2021bilinear}
Du, S.~S., Kakade, S.~M., Lee, J.~D., Lovett, S., Mahajan, G., Sun, W., and
  Wang, R. (2021).
\newblock Bilinear classes: A structural framework for provable generalization
  in {RL}.
\newblock {\em arXiv preprint arXiv:2103.10897}.

\bibitem[Du et~al., 2020a]{du2019good}
Du, S.~S., Kakade, S.~M., Wang, R., and Yang, L.~F. (2020a).
\newblock Is a good representation sufficient for sample efficient
  reinforcement learning?
\newblock In {\em International Conference on Learning Representations}.

\bibitem[Du et~al., 2020b]{du2020agnostic}
Du, S.~S., Lee, J.~D., Mahajan, G., and Wang, R. (2020b).
\newblock Agnostic {Q}-learning with function approximation in deterministic
  systems: Tight bounds on approximation error and sample complexity.
\newblock {\em Neural Information Processing Systems}.

\bibitem[Du et~al., 2019]{du2019provably}
Du, S.~S., Luo, Y., Wang, R., and Zhang, H. (2019).
\newblock Provably efficient {Q}-learning with function approximation via
  distribution shift error checking oracle.
\newblock In {\em Advances in Neural Information Processing Systems}, pages
  8058--8068.

\bibitem[Even-Dar and Mansour, 2003]{even2003learning}
Even-Dar, E. and Mansour, Y. (2003).
\newblock Learning rates for {Q}-learning.
\newblock {\em Journal of machine learning Research}, 5(Dec):1--25.

\bibitem[Hao et~al., 2020]{hao2020sparse}
Hao, B., Duan, Y., Lattimore, T., Szepesv{\'a}ri, C., and Wang, M. (2020).
\newblock Sparse feature selection makes batch reinforcement learning more
  sample efficient.
\newblock {\em arXiv preprint arXiv:2011.04019}.

\bibitem[He et~al., 2020]{he2020logarithmic}
He, J., Zhou, D., and Gu, Q. (2020).
\newblock Logarithmic regret for reinforcement learning with linear function
  approximation.
\newblock {\em arXiv preprint arXiv:2011.11566}.

\bibitem[Jaksch et~al., 2010]{jaksch2010near}
Jaksch, T., Ortner, R., and Auer, P. (2010).
\newblock Near-optimal regret bounds for reinforcement learning.
\newblock {\em Journal of Machine Learning Research}, 11(4).

\bibitem[Jiang et~al., 2017]{jiang2017contextual}
Jiang, N., Krishnamurthy, A., Agarwal, A., Langford, J., and Schapire, R.~E.
  (2017).
\newblock Contextual decision processes with low {B}ellman rank are
  {PAC}-learnable.
\newblock In {\em International Conference on Machine Learning}, pages
  1704--1713. PMLR.

\bibitem[Jin et~al., 2018]{jin2018q}
Jin, C., Allen-Zhu, Z., Bubeck, S., and Jordan, M.~I. (2018).
\newblock Is {Q}-learning provably efficient?
\newblock In {\em Advances in Neural Information Processing Systems}, pages
  4863--4873.

\bibitem[Jin et~al., 2021]{jin2021bellman}
Jin, C., Liu, Q., and Miryoosefi, S. (2021).
\newblock Bellman {E}luder dimension: New rich classes of {RL} problems, and
  sample-efficient algorithms.
\newblock {\em arXiv preprint arXiv:2102.00815}.

\bibitem[Jin et~al., 2020]{jin2020provably}
Jin, C., Yang, Z., Wang, Z., and Jordan, M.~I. (2020).
\newblock Provably efficient reinforcement learning with linear function
  approximation.
\newblock In {\em Conference on Learning Theory}, pages 2137--2143. PMLR.

\bibitem[Kakade, 2003]{kakade2003sample}
Kakade, S. (2003).
\newblock {\em On the sample complexity of reinforcement learning}.
\newblock PhD thesis, University of London.

\bibitem[Kearns and Singh, 1999]{kearns1999finite}
Kearns, M.~J. and Singh, S.~P. (1999).
\newblock Finite-sample convergence rates for {Q}-learning and indirect
  algorithms.
\newblock In {\em Advances in neural information processing systems}, pages
  996--1002.

\bibitem[Lattimore and Szepesv{\'a}ri, 2020]{lattimore2020bandit}
Lattimore, T. and Szepesv{\'a}ri, C. (2020).
\newblock {\em Bandit algorithms}.
\newblock Cambridge University Press.

\bibitem[Lattimore et~al., 2020]{lattimore2020learning}
Lattimore, T., Szepesvari, C., and Weisz, G. (2020).
\newblock Learning with good feature representations in bandits and in {RL}
  with a generative model.
\newblock In {\em International Conference on Machine Learning}, pages
  5662--5670. PMLR.

\bibitem[Li et~al., 2021a]{li2021tightening}
Li, G., Cai, C., Chen, Y., Gu, Y., Wei, Y., and Chi, Y. (2021a).
\newblock Is {Q}-learning minimax optimal? a tight sample complexity analysis.
\newblock {\em arXiv preprint arXiv:2102.06548}.

\bibitem[Li et~al., 2021b]{li2021breaking}
Li, G., Shi, L., Chen, Y., Gu, Y., and Chi, Y. (2021b).
\newblock Breaking the sample complexity barrier to regret-optimal model-free
  reinforcement learning.
\newblock {\em accepted to Neural Information Processing Systems (NeurIPS)}.

\bibitem[Li et~al., 2020a]{li2020breaking}
Li, G., Wei, Y., Chi, Y., Gu, Y., and Chen, Y. (2020a).
\newblock Breaking the sample size barrier in model-based reinforcement
  learning with a generative model.
\newblock {\em Advances in Neural Information Processing Systems}, 33.

\bibitem[Li et~al., 2020b]{li2020sample}
Li, G., Wei, Y., Chi, Y., Gu, Y., and Chen, Y. (2020b).
\newblock Sample complexity of asynchronous {Q}-learning: Sharper analysis and
  variance reduction.
\newblock {\em Advances in Neural Information Processing Systems (NeurIPS)}.

\bibitem[Long and Han, 2021]{long20212}
Long, J. and Han, J. (2021).
\newblock An $ l^2$ analysis of reinforcement learning in high dimensions with
  kernel and neural network approximation.
\newblock {\em arXiv preprint arXiv:2104.07794}.

\bibitem[Modi et~al., 2021]{modi2021model}
Modi, A., Chen, J., Krishnamurthy, A., Jiang, N., and Agarwal, A. (2021).
\newblock Model-free representation learning and exploration in low-rank
  {MDP}s.
\newblock {\em arXiv preprint arXiv:2102.07035}.

\bibitem[Munos, 2005]{munos2005error}
Munos, R. (2005).
\newblock Error bounds for approximate value iteration.
\newblock In {\em Proceedings of the National Conference on Artificial
  Intelligence}, volume~20, pages 1006--1011.

\bibitem[Osband and Van~Roy, 2014]{osband2014model}
Osband, I. and Van~Roy, B. (2014).
\newblock Model-based reinforcement learning and the {E}luder dimension.
\newblock In {\em Proceedings of the 27th International Conference on Neural
  Information Processing Systems-Volume 1}, pages 1466--1474.

\bibitem[Osband et~al., 2016]{osband2016generalization}
Osband, I., Van~Roy, B., and Wen, Z. (2016).
\newblock Generalization and exploration via randomized value functions.
\newblock In {\em International Conference on Machine Learning}, pages
  2377--2386. PMLR.

\bibitem[Shariff and Szepesv{\'a}ri, 2020]{shariff2020efficient}
Shariff, R. and Szepesv{\'a}ri, C. (2020).
\newblock Efficient planning in large {MDP}s with weak linear function
  approximation.
\newblock {\em arXiv preprint arXiv:2007.06184}.

\bibitem[Sidford et~al., 2018]{sidford2018variance}
Sidford, A., Wang, M., Wu, X., and Ye, Y. (2018).
\newblock Variance reduced value iteration and faster algorithms for solving
  {M}arkov decision processes.
\newblock In {\em Proceedings of the Twenty-Ninth Annual ACM-SIAM Symposium on
  Discrete Algorithms}, pages 770--787. SIAM.

\bibitem[Silver et~al., 2016]{silver2016mastering}
Silver, D., Huang, A., Maddison, C.~J., Guez, A., Sifre, L., Van Den~Driessche,
  G., Schrittwieser, J., Antonoglou, I., Panneershelvam, V., Lanctot, M.,
  et~al. (2016).
\newblock Mastering the game of {G}o with deep neural networks and tree search.
\newblock {\em Nature}, 529(7587):484--489.

\bibitem[Simchowitz and Jamieson, 2019]{simchowitz2019non}
Simchowitz, M. and Jamieson, K. (2019).
\newblock Non-asymptotic gap-dependent regret bounds for tabular {MDP}s.
\newblock {\em arXiv preprint arXiv:1905.03814}.

\bibitem[Touati and Vincent, 2020]{touati2020efficient}
Touati, A. and Vincent, P. (2020).
\newblock Efficient learning in non-stationary linear {M}arkov decision
  processes.
\newblock {\em arXiv preprint arXiv:2010.12870}.

\bibitem[Wainwright, 2019]{wainwright2019stochastic}
Wainwright, M.~J. (2019).
\newblock Stochastic approximation with cone-contractive operators: Sharp
  $\ell_{\infty}$-bounds for {Q}-learning.
\newblock {\em arXiv preprint arXiv:1905.06265}.

\bibitem[Wang et~al., 2021a]{wang2021sample}
Wang, B., Yan, Y., and Fan, J. (2021a).
\newblock Sample-efficient reinforcement learning for linearly-parameterized
  mdps with a generative model.
\newblock {\em arXiv preprint arXiv:2105.14016}.

\bibitem[Wang et~al., 2020a]{wang2020reward}
Wang, R., Du, S.~S., Yang, L.~F., and Salakhutdinov, R. (2020a).
\newblock On reward-free reinforcement learning with linear function
  approximation.
\newblock {\em arXiv preprint arXiv:2006.11274}.

\bibitem[Wang et~al., 2020b]{wang2020reinforcement}
Wang, R., Salakhutdinov, R.~R., and Yang, L. (2020b).
\newblock Reinforcement learning with general value function approximation:
  Provably efficient approach via bounded {E}luder dimension.
\newblock {\em Advances in Neural Information Processing Systems}, 33.

\bibitem[Wang et~al., 2019]{wang2019optimism}
Wang, Y., Wang, R., Du, S.~S., and Krishnamurthy, A. (2019).
\newblock Optimism in reinforcement learning with generalized linear function
  approximation.
\newblock {\em arXiv preprint arXiv:1912.04136}.

\bibitem[Wang et~al., 2021b]{wang2021exponential}
Wang, Y., Wang, R., and Kakade, S.~M. (2021b).
\newblock An exponential lower bound for linearly-realizable {MDP}s with
  constant suboptimality gap.
\newblock {\em arXiv preprint arXiv:2103.12690}.

\bibitem[Wei et~al., 2021]{wei2021learning}
Wei, C.-Y., Jahromi, M.~J., Luo, H., and Jain, R. (2021).
\newblock Learning infinite-horizon average-reward {MDP}s with linear function
  approximation.
\newblock In {\em International Conference on Artificial Intelligence and
  Statistics}, pages 3007--3015.

\bibitem[Weisz et~al., 2021a]{weisz2021query}
Weisz, G., Amortila, P., Janzer, B., Abbasi-Yadkori, Y., Jiang, N., and
  Szepesv{\'a}ri, C. (2021a).
\newblock On query-efficient planning in {MDP}s under linear realizability of
  the optimal state-value function.
\newblock {\em arXiv preprint arXiv:2102.02049}.

\bibitem[Weisz et~al., 2021b]{weisz2021exponential}
Weisz, G., Amortila, P., and Szepesv{\'a}ri, C. (2021b).
\newblock Exponential lower bounds for planning in {MDP}s with
  linearly-realizable optimal action-value functions.
\newblock In {\em Algorithmic Learning Theory}, pages 1237--1264. PMLR.

\bibitem[Wen and Van~Roy, 2017]{wen2017efficient}
Wen, Z. and Van~Roy, B. (2017).
\newblock Efficient reinforcement learning in deterministic systems with value
  function generalization.
\newblock {\em Mathematics of Operations Research}, 42(3):762--782.

\bibitem[Yang et~al., 2021]{yang2021q}
Yang, K., Yang, L., and Du, S. (2021).
\newblock Q-learning with logarithmic regret.
\newblock In {\em International Conference on Artificial Intelligence and
  Statistics}, pages 1576--1584. PMLR.

\bibitem[Yang and Wang, 2019]{yang2019sample}
Yang, L. and Wang, M. (2019).
\newblock Sample-optimal parametric {Q}-learning using linearly additive
  features.
\newblock In {\em International Conference on Machine Learning}, pages
  6995--7004.

\bibitem[Yang and Wang, 2020]{yang2020reinforcement}
Yang, L. and Wang, M. (2020).
\newblock Reinforcement learning in feature space: Matrix bandit, kernels, and
  regret bound.
\newblock In {\em International Conference on Machine Learning}, pages
  10746--10756. PMLR.

\bibitem[Yang et~al., 2020]{yang2020bridging}
Yang, Z., Jin, C., Wang, Z., Wang, M., and Jordan, M.~I. (2020).
\newblock Bridging exploration and general function approximation in
  reinforcement learning: Provably efficient kernel and neural value
  iterations.
\newblock {\em arXiv preprint arXiv:2011.04622}.

\bibitem[Yin et~al., 2021]{yin2021efficient}
Yin, D., Hao, B., Abbasi-Yadkori, Y., Lazi{\'c}, N., and Szepesv{\'a}ri, C.
  (2021).
\newblock Efficient local planning with linear function approximation.
\newblock {\em arXiv preprint arXiv:2108.05533}.

\bibitem[Zanette et~al., 2020a]{zanette2020frequentist}
Zanette, A., Brandfonbrener, D., Brunskill, E., Pirotta, M., and Lazaric, A.
  (2020a).
\newblock Frequentist regret bounds for randomized least-squares value
  iteration.
\newblock In {\em International Conference on Artificial Intelligence and
  Statistics}, pages 1954--1964. PMLR.

\bibitem[Zanette et~al., 2019]{zanette2019limiting}
Zanette, A., Lazaric, A., Kochenderfer, M., and Brunskill, E. (2019).
\newblock Limiting extrapolation in linear approximate value iteration.
\newblock {\em Neural Information Processing Systems}.

\bibitem[Zanette et~al., 2020b]{zanette2020learning}
Zanette, A., Lazaric, A., Kochenderfer, M., and Brunskill, E. (2020b).
\newblock Learning near optimal policies with low inherent {B}ellman error.
\newblock In {\em International Conference on Machine Learning}, pages
  10978--10989. PMLR.

\bibitem[Zhang et~al., 2020]{zhang2020almost}
Zhang, Z., Zhou, Y., and Ji, X. (2020).
\newblock Almost optimal model-free reinforcement learning via
  reference-advantage decomposition.
\newblock {\em Advances in Neural Information Processing Systems}, 33.

\bibitem[Zhou et~al., 2020]{zhou2020nearly}
Zhou, D., Gu, Q., and Szepesvari, C. (2020).
\newblock Nearly minimax optimal reinforcement learning for linear mixture
  {M}arkov decision processes.
\newblock {\em arXiv preprint arXiv:2012.08507}.

\end{thebibliography}
